\documentclass[10pt]{article} 
\usepackage[accepted]{tmlr}

\usepackage{amsmath,amsfonts,bm}

\def\eqref#1{equation~\ref{#1}}

\def\1{\bm{1}}

\DeclareMathAlphabet{\mathsfit}{\encodingdefault}{\sfdefault}{m}{sl}
\SetMathAlphabet{\mathsfit}{bold}{\encodingdefault}{\sfdefault}{bx}{n}

\usepackage{url}
\usepackage{enumitem}
\usepackage{array}
\usepackage{xspace}
\usepackage{amsmath,amsthm,amssymb}
\usepackage{graphicx}
\usepackage[normalem]{ulem}
\usepackage{hyperref}
\usepackage{cleveref}

\newcommand*{\cD}{\mathcal{D}}
\newcommand*{\cR}{\mathcal{R}}
\newcommand*{\cC}{\mathcal{C}}
\newcommand*{\NN}{\mathbb{N}}
\newcommand*{\RR}{\mathbb{R}}
\newcommand{\jac}{\mathrm{Jac}\ }
\newcommand{\crit}{\mathrm{crit }}
\newcommand{\rank}{\mathrm{rank}}
\newcommand{\proj}{\mathrm{proj}}
\newcommand{\graph}{\mathrm{graph}\,}
\newcommand{\sa}{definable\xspace}

\theoremstyle{plain}
\newtheorem{theorem}{Theorem}[section]

\newtheorem{lemma}[theorem]{Lemma}

\newtheorem{proposition}[theorem]{Proposition}

\newtheorem*{theorem*}{Theorem}
\newtheorem*{proposition*}{Proposition}

\theoremstyle{definition}
\newtheorem{definition}[theorem]{Definition}
\newtheorem{remark}[theorem]{Remark}
\newtheorem{example}[theorem]{Example}

\newtheorem*{definition*}{Definition}

\title{Convergence of optimizers implies eigenvalue filtering at equilibrium}

\author{\name Jérôme Bolte \email jbolte@ut-capitole.fr  \\
      \addr Toulouse School of Economics, University of Toulouse Capitole, Toulouse, France\\
      \AND
      \name Quoc-Tung Le \email quoc-tung.le@univ-grenoble-alpes.fr \\
      \addr Universit\'e Grenoble Alpes, LJK, CNRS, Grenoble INP, 38000 Grenoble, France
      \AND
      \name Edouard Pauwels \email edouard2.pauwels@ut-capitole.fr \\
      \addr Toulouse School of Economics, University of Toulouse Capitole, Toulouse, France. \\
}

\def\month{11}  
\def\year{2026} 
\def\openreview{\url{https://openreview.net/forum?id=3tuxYlFWTy}} 

\begin{document}

\maketitle

\begin{abstract}%
    Ample empirical evidence in deep neural network training suggests that a variety of optimizers tend to find nearly global optima. In this article, we adopt the reversed perspective that convergence to an arbitrary point is assumed rather than proven, focusing on the consequences of this assumption. From this viewpoint, in line with recent advances on the edge-of-stability phenomenon, we argue that different optimizers effectively act as eigenvalue filters determined by their hyperparameters: 
    Specifically, the standard gradient descent method {excludes limiting critical points whose largest Hessian eigenvalue exceeds twice the inverse step size, whereas the smooth, unnormalized variant of sharpness-aware minimization (USAM) enforces a stricter stability bound.} 
    Inspired by these insights, we propose two {variants of USAM that enforce even stricter stability bounds and further illustrate our general theory.} 
    Our theoretical analysis leverages a generalized Hadamard--Perron stable manifold theorem and applies to general \sa $C^2$ functions, without requiring additional non-degeneracy conditions or global Lipschitz bound assumptions. We support our conclusions with numerical experiments on feed-forward neural networks.
\end{abstract}

\section{Introduction}

The stability of optimization algorithms has emerged as a key factor in understanding both the training dynamics and generalization of deep neural networks \citep{wu2018how,cohen2021gradient}. Stable training is correlated with desirable properties such as wide attraction basins and flat minima, which relate to generalization performances. This perspective underlies recent developments like sharpness-aware minimization (SAM) for finding flatter solutions \citep{foret2021sharpnessaware}, as well as empirical and practical analyses of the implicit bias of gradient methods toward low-curvature minima \citep{mulayoff2021implicit}. Moreover, the empirically observed ``edge of stability'' phenomenon for gradient descent and its variants \citep{cohen2021gradient,cohen2022adaptive,kaur2022maximum,andreyev2024edge} has highlighted that standard training often operates near the boundary of stability (e.g. learning rates are eventually close to the largest stable value). Theoretical results in dynamical systems and optimization further show that, under generic conditions, gradient-based algorithms avoid strict saddle points \cite{lee2016gradient,panageas2017gradient,ahn2022understanding}. Collectively, these observations indicate that stability plays a crucial role in where and how training converges \citep{ahn2022understanding}.

On the other hand, massive engineering efforts and improved heuristics over the past decade have made successful training almost the norm in deep learning practice, provided hyperparameters are well-tuned. This “systematic” convergence, and its tight link with the geometry of the loss landscape, invites a shift in perspective: rather than asking under what conditions an algorithm will converge, we ask \emph{why a given successful training run did converge} and how the choice of hyperparameters made that possible.

To understand these phenomena in a unified way, we model optimizers by the dynamics of the type
\begin{align}
x_{k+1} = G_\alpha(x_k) = D x_k - \alpha g(x_k)  \qquad k = 0,1,2,\ldots,
\label{eq:iterative-methods}
\end{align}
where $D \in \mathbb{R}^{m\times m}$ is invertible, $g: \mathbb{R}^m \to \mathbb{R}^m$ is a $C^1$ continuously differentiable and \emph{\sa}\footnote{More details about \sa functions are available in \Cref{subsec:definable}} mapping (for example, $g$ could be the gradient of a \sa loss function), and $\alpha > 0$ is a scalar step size. $D$ and $g$ may depend additionally on some fixed hyperparameters $p \in \RR^\ell$; we single out the hyperparameter $\alpha$ to emphasize its role as the “learning rate”, fundamental in practice. This formulation is quite general and covers many common methods (gradient descent, heavy ball method, {the unnormalized version of sharpness-aware optimization} (USAM)) by an appropriate choice of $D$ and $g$.
As explained above, we now focus on the {\em regime of successful runs} with a generalized form of the Hadamard--Perron stable manifold theorem:

\begin{theorem}[Successful runs imply nonexpansiveness at equilibrium] \label{th:abstract}
Let $D \in \mathbb{R}^{m \times m}$ be an invertible matrix, $g: \mathbb{R}^m \to \mathbb{R}^m$ be a $C^1$ \sa mapping. For almost all  $x_0 \in \mathbb{R}^m$ and $\alpha>0$ the following assertion holds true: if the sequence $(x_k)_{k\in \mathbb{N}}$ converges to some point $\bar{x}$, then the spectral radius of the Jacobian of $D-\alpha g$ at $\bar x$ is at most $1$.
\end{theorem}
\newenvironment{theoremprime}
  {\begingroup
   \renewcommand{\thetheorem}{\arabic{theorem}'}
   \addtocounter{theorem}{-1}
   \begin{theorem}}
  {\end{theorem}\endgroup}

With obvious notation, the conclusion reads $\rho\,\bigl(\mathrm{Jac}\, G_\alpha(\bar x)\bigr)\le1$.
This is a partial converse to the well-known local stability statement: when 
$\rho\,(\mathrm{Jac}\, G_\alpha(\bar x))<1$,
 any sequence initialized sufficiently close to \(\bar x\) converges. 
Now, if the dependence on hyperparameters is made explicit through 
\(D=D(p)\), \(g(x)=g(x,p)\), 
and if the dynamics is built upon the gradient of a loss \(f\), then the inequality $\rho\,\bigl(\mathrm{Jac}\, G_\alpha(\bar x; p)\bigr)\;\le\;1$
unfolds into ``algebraic relations" tying the eigenvalues of \(\nabla^2 f(\bar x)\) to the hyperparameters \((\alpha,p)\). 
{The fact that for almost every initialization and step sizes, the optimizer excludes limiting critical points whose local curvature given by the spectrum of $\nabla^2 f(\bar{x})$ violates the stability bound (depending on $\alpha$ and $p$)}
-- we call this phenomenon {\em eigenvalue filtering}. 

\Cref{th:abstract} could also be seen as a one-sided version of the edge-of-stability phenomenon in deep learning. The latter is the empirical observation that the spectral radius bound tends to get saturated along optimizers sequences and, in the long run, holds roughly as an equality \citep{cohen2021gradient,cohen2022adaptive,kaur2022maximum,andreyev2024edge}. Our result proves one inequality: with probability one, converging sequences need to find step regimes with spectral radius lower than one. A rigorous proof of the reverse inequality, showing therefore {\em saturation} of the radius at the value $1$ for neural networks, remains open.

As an intuition-building example, consider plain gradient descent (GD) on a \(C^2\) loss. 
Classical local analysis around a nondegenerate minimum with Hessian eigenvalue \(\lambda\) shows that convergence requires roughly \(0<\alpha<2/\lambda\); otherwise, iterates oscillate or diverge. 
Conversely, if we place ourselves in a scenario where GD \emph{does} converge, then necessarily \(\lambda\leq 2/\alpha\) for all Hessian eigenvalues at the limit point. 
In other words, the reached points must satisfy the curvature bound \(\lambda\leq 2/\alpha\): the method {excludes limiting critical points whose local curvature violates the stability bound.} 
Whether convergence is guaranteed a priori, and to what extent this is realistic in full generality, remains open; nevertheless, deep learning offers a surprisingly rich empirical field that lends substantial support to the setting considered here \citep{cohen2021gradient,cohen2022adaptive,kaur2022maximum,andreyev2024edge}. Note that \Cref{th:abstract} considerably extends the main result of \cite{ahn2022understanding} considered in this paragraph, note as well that it has the following alternative  formulation by contraposition:
\begin{theoremprime}[Spectral radius and nonconvergence]
    \label{th:contraposition-main-thm}
    Let $D,g$ be as in \Cref{th:abstract}. For generic step size $\alpha > 0$ and random initialization from a density, the probability of converging to a point $\bar x$ such that $\rho(\jac G_\alpha(\bar{x})) > 1$ is zero.
\end{theoremprime}

{\Cref{th:contraposition-main-thm} makes clearer the true nature of \Cref{th:abstract}: optimizers exclude convergence to certain points, namely limiting critical points whose local curvature violates the stability bound, for almost all initializations and step sizes. As such, one should not understand the convergence in \Cref{th:abstract} as an assumption, but as an abstract condition which we use to account for the successful training regime.}


To illustrate our findings, we describe the relation between Hessian eigenvalues and hyperparameters for several optimization algorithms in terms of stable convergence: gradient descent, the heavy ball method, Nesterov's accelerated gradient method (with constant momentum). We push the investigation further in the context of sharpness aware minimization \cite{foret2021sharpnessaware}, whose goal is to design recursive algorithms in the form of \Cref{eq:iterative-methods} which tend to favor local minima with lower curvature. We focus on its un-normalized algorithmic variant, USAM \cite{andriushchenko22towards,dai2023crucial}, since the original SAM iteration are not smooth (not even continuous, due to normalization). {Despite the original SAM being out of the scope of \Cref{th:abstract}, our analysis of its USAM variant is of interest of its own since it has been argued in the literature that USAM ``captures the central behavior of SAM in training dynamics and generalization'' \cite{andriushchenko22towards,zhou2025sam}
, Furthermore, our assumptions are significantly relaxed in comparison to this literature.}. Our analysis reveals that, the USAM algorithm induces more constraints on the limiting Hessian eigenvalues, which is consistent with the study of a simplified version of USAM in \cite{zhou2025sam}. To further illustrate this idea, we design two {illustrative USAM} 
variants — \emph{Two-step USAM} and \emph{Hessian USAM} — which incorporate, respectively, an extra ascent step and second-order information into the SAM update. Our analysis predicts that these variants enforce \emph{stricter} eigenvalue constraints under the convergence regime.
We confirm these predictions empirically with numerical experiments on a multi-layer perceptron with MNIST and FASHION-MNIST datasets, as well as a wide ResNet architecture with the CIFAR10 dataset. These experiments qualitatively align with the prediction of the theory {regarding the eigenvalue filtering phenomenon}. {On the other hand, our experiments are not sufficient to conclude about a practical advantage of either Two-step USAM or Hessian USAM, compared to the vanilla SAM algorithm. Let us also mention that our theory covers the deterministic setting while our experiments include subsampling noise, a gap which is to be explored in future research.}

\subsection{Related work}
\Cref{tab:summary-results} provides a review of the assumptions for eigenvalue filtering results: lower bounds presented in \cite{lee2016gradient,panageas2017gradient,o2019behavior,sun2019heavy} and upper bounds in \cite{cohen2021gradient,ahn2022understanding,agarwala2023sam,zhou2025sam}. In contrast, our assumption is $C^2$ and definable in most of the cases {(except \Cref{cor:application-hsam} where we need $f$ to be $C^3$)}.

\begin{table}[bhtp] 
\begin{center}
\renewcommand{\arraystretch}{1.5}
\footnotesize
\setlength{\tabcolsep}{2pt}
\begin{tabular}{
    >{\raggedright\arraybackslash}p{0.18\textwidth}
    >{\raggedright\arraybackslash}p{0.40\textwidth}
    >{\raggedright\arraybackslash}p{0.40\textwidth}
}
    \hline
    \textbf{Algorithms} 
    & \textbf{Assumptions for lower bounds (SOTA)} 
    & \textbf{Assumptions for upper bounds (SOTA)} \\
    \hline
    Gradient &
    $C^2$, Lipschitz gradient, small step sizes -- \cite{lee2016gradient,panageas2017gradient}. &
    Lipschitz gradient, abstract step-size constraints -- \cite{ahn2022understanding}. \\
    
    Heavy ball &
    $C^2$, Lipschitz gradient, small step sizes -- \cite{sun2019heavy,o2019behavior}. &
    Quadratic objectives -- \cite{cohen2021gradient}. \\
    
    Nesterov &
    Quadratic objectives -- \cite{o2019behavior}. &
    Quadratic objectives -- \cite{cohen2021gradient}. \\
    
    SAM &
    No result to the best of our knowledge. &
    Least-squares regression  \cite{agarwala2023sam,zhou2025sam}. \\
    \hline
    \textbf{Algorithms} & \textbf{Our assumptions} &  \\
    \hline
    All of the above; any method as in \eqref{eq:iterative-methods} & 
    \multicolumn{2}{p{0.80\textwidth}}{
    $C^2$ and \emph{definable}\footnote{For example, any function obtained by compositions of affine mappings, exponentials, logarithms, and piecewise polynomials.}
    (see \Cref{th:abstract,th:contraposition-main-thm}). 
    \smallskip

    Implications for the specific algorithms listed above are detailed in 
    \Cref{cor:application-gradient-descent,cor:application-momentum-gradient,cor:application-nesterov-gradient,cor:application-usam-momentum}.} \\
    \hline
\end{tabular}
\end{center}
\caption{Comparison between existing lower and upper bound results for several first-order methods and our unified framework. 
Previous works typically require strong structural assumptions such as Lipschitz continuity of the gradient, small step sizes, or quadratic objectives (see also \Cref{sec:proofSketch} for details). 
In contrast, our approach assumes only $C^2$ regularity and definability, thereby recovering known results while substantially relaxing classical hypotheses. 
Moreover, our results provide new guarantees for certain methods (e.g., SAM) for which no comparable lower bound was previously available, to the best of our knowledge.}
\label{tab:summary-results}
\end{table}

\cite{anosov1967geodesic} attributes the stable manifold theorem to \cite{hadamard1901iteration} (see \cite{hasselblatt2017iteration}) and \cite{perron1929stabilitat} acknowledging earlier versions by Darboux, Poincaré and Lyapunov. Generic avoidance of strict saddle points by gradient flows dates back at least to \cite{thom1949partition}. In an optimization context, similar ideas have been used for stochastic algorithms \citep{pemantle1990nonconvergence}, inertial dynamics \citep{goudou2009gradient}, and more recently for the gradient algorithm in machine learning \cite{lee2016gradient,panageas2017gradient,ahn2022understanding}.

Implicit bias toward flat minima through stability is a common theme in the neural network literature \citep{wu2018how,ahn2022understanding}, with connection to generalization \cite{mulayoff2021implicit,qiao2024stable,wu2025taking,kaur2022maximum}.
This motivated the development of sharpness aware minimization algorithms \citep{foret2021sharpnessaware,andriushchenko22towards,dai2023crucial} with several follow-up works on the connection between these approaches, flat minima, and prediction generalization
\citep{andriushchenko2023modern,marion2024deep,agarwala2023sam,zhou2025sam,tan2024stabilizing}

The edge-of-stability phenomenon is the empirical observation that deep network training tends to saturate the ``convergence stability constraints" on Hessian eigenvalues. This includes gradient descent
\citep{cohen2021gradient}, adaptive methods \citep{cohen2022adaptive} and stochastic algorithms \cite{andreyev2024edge,agarwala2024high}.
This was studied theoretically for logistic regression \citep{wu2024large,wu2023implicit} and in broader non convex optimization contexts \citep{damian2023selfstabilization,arora2022understanding,ahn2022understanding}. The closest results to our main theorem are given in \cite{lee2016gradient,panageas2017gradient,ahn2022understanding}. We consider a much broader class of algorithms and under very mild hypotheses, effectively removing abstract non-degeneracy conditions or global Lipschitz bound assumptions.

After submitting the initial version of this work, we became aware of the independent work \cite{mucsat2025gradient} which describes a similar abstract unstable fixed point avoidance without the \sa assumptions but with a less precise characterization of initializations and step sizes to be avoided. This is used to justify avoidance of strict saddle for the gradient algorithm with backtracking line-search. {Another concurrent work is \cite{dennis2025characterizing} which studies the stochastic, overparameterized least square regression.}

\section{Large step analysis of iterative methods}
\label{sec:largeStepSize}
Our main stability result is stated as follows, this is a precise version of \Cref{th:abstract} and \Cref{th:contraposition-main-thm}.
\begin{theorem}[Sharpened version of \Cref{th:abstract} and \Cref{th:contraposition-main-thm}]
    \label{theorem:learning-rate-gradient-stability}
	Let $D \in \RR^{m \times m}$ be an invertible matrix, $g \colon \RR^m \to \RR^m$ be a $C^1$ and \sa function and consider the recursion \Cref{eq:iterative-methods}.
    There exists $\Lambda \subset \RR_+$, whose complement is finite, such that for any $\alpha \in \Lambda$, the following set is contained in a countable union of $C^1$ submanifolds\footnote{Without further precisions, in the main text, all submanifolds are supposed to be embedded.} of dimension at most $m-1$: 
	\begin{align*}
		W_\alpha = \{ x_0 \in \RR^m \mid \exists\ \bar{x} \text{ s.t. } G_\alpha(\bar{x}) = \bar{x},\, \rho(\jac G_\alpha(\bar{x})) > 1,\, x_k \to \bar{x},\, k \to \infty\}
	\end{align*}
\end{theorem}
From \Cref{theorem:learning-rate-gradient-stability} we also have that for almost all $\alpha,x_0$, if the limit exists, then the corresponding spectral radius is at most $1$. This form is \Cref{th:abstract} in the introduction, a result that we will use repeatedly. 
In \Cref{theorem:learning-rate-gradient-stability}, discarding a subset of step sizes and initializations is necessary as illustrated in \Cref{ex:genericityRequired} in \Cref{sec:additionalResults}.
\Cref{theorem:learning-rate-gradient-stability} extends considerably  \cite[Theorem 1]{ahn2022understanding}, which was stated for the gradient with topological assumptions that we do not need. Moreover, our approach encompasses abstract dynamics and come with {\em easily verifiable assumptions}. 
Furthermore, the abstract form of \Cref{theorem:learning-rate-gradient-stability} is more general.
The proof of \Cref{theorem:learning-rate-gradient-stability} leverages strong rigidity properties of \sa maps and a stable manifold theorem, as presented in the two following subsections. 
A variant of \Cref{theorem:learning-rate-gradient-stability} is proved concurrently in \cite{mucsat2025gradient} for gradient descent without definable assumptions (but with less precise characterizations of the set $\Lambda$ and $W_\alpha$).

\subsection{A stable manifold theorem beyond local diffeomorphisms}
Stability results similar to \Cref{theorem:learning-rate-gradient-stability} are numerous \citep{pemantle1990nonconvergence,goudou2009gradient,lee2016gradient,panageas2017gradient,ahn2022understanding}, they rely on variations of the Hadamard--Perron theorem. In dynamical systems theory, these are typically presented for local diffeomorphisms \cite{hirsch1977invariant,shub1987global}. As presented in \Cref{ex:genericityRequired}, being a local diffeomorphism may fail for general step size $\alpha$ as considered in \Cref{theorem:learning-rate-gradient-stability}.

It is actually known in dynamical systems literature that center stable manifold theorems hold beyond local diffeomorphisms, without requiring invertibility, as seen in the following result.
\begin{theorem}[Refined version of stable center manifold theorem]
    \label{theorem:refined-center-stable-manifold}
    Let $p$ be a fixed point for the $C^1$ function $F: U \to \RR^n$ where $U \subseteq \RR^n$ is an open neighborhood of $p$ in $\RR^n$. Let $E_{sc} \oplus E_u$ be the invariant splitting of $\RR^n$ into generalized eigenspaces of $\jac F(p)$ corresponding to the eigenvalues of absolute value less or equal to $1$, and strictly greater than $1$ respectively. To the $\jac F(p)$ invariant subspace $E_{sc}$, there is an associated local $F$ invariant $C^1$ submanifold $W^{\text{sc}}_{\text{loc}}$ of dimension $\dim (E_{sc})$, and a ball $B$ around $p$ such that:
    $F(W^{\text{sc}}_{\text{loc}}) \cap B \subseteq W^{\text{sc}}_{\text{loc}}$, 
    and if $F^k(x) \in B$ for all $k \geq 0$, then $x \in W^{\text{sc}}_{\text{loc}}$.
\end{theorem}
\Cref{theorem:refined-center-stable-manifold} has already been presented in the literature with sketched proofs. We provide a detailed and self-contained proof for completeness in \Cref{sec:center-stable-manifold} (see also \cite[Section 2.2]{mucsat2025gradient}).
\begin{itemize}
    \item In \cite{shub1987global} chapter 5, appendix III., 
    there is a remark following the statement of Theorem III.2 to justify the existence of a center stable manifold when $F$ is a diffeomorphism. Then Exercise III.2 states that the invertibility of $F$ is actually not necessary.
    \item Similarly, in \cite{hirsch1977invariant}, Theorem 5A.3 states a result about the existence of a center unstable manifold. A remark follows justifying the existence of a center stable manifold if $F$ is a diffeomorphism. The last paragraph of Section 5 from this book provides a quick justification of the fact that the invertibility of $F$ is not necessary.
\end{itemize}

\subsection{Definable objects and their scope}
\label{subsec:definable}
\Cref{th:abstract} is stated under definable assumptions and uses this as a versatile sufficient condition.
We refer to \cite{coste2000introduction,coste2000introductionSA} for an introductory exposition of semi-algebraic and definable geometry as well as \cite{attouch2010proximal,attouch2013Convergence} for numerous examples in optimization. In the following, we only justify why the definable assumptions are satisfied by virtually all functions in the learning context. Therefore, to apply our result in practice, it suffices to verify \emph{only} the smoothness condition.

\begin{example}[Examples of definable functions]
    \label{ex:definable-functions}
    {\rm 
    Affine, exponential, logarithm, polynomial, rational, square root, relu, matrix rank, $\ell_p$ norms, spectral norm, maximum coordinate, argmax coordinate, sorting operation, clipping, softmax etc. \ldots}
\end{example}

\begin{proposition}[Properties of \sa functions]
    \label{prop:composition-differentiation}
    The set of \sa functions is closed under composition and differentiation operations.
\end{proposition}

Due to \Cref{ex:definable-functions} and \Cref{prop:composition-differentiation}, the class of \sa functions is very well adapted to study deep networks, which are parameterized compositions. The training loss of a deep network built with definable functions, \textit{e.g.} a relu multilayer {perceptron}, attention layers, followed by the squared or the cross-entropy loss, is definable. Moreover, due to \Cref{prop:composition-differentiation}, the gradient and Hessian of definable functions are also definable. Therefore, most optimization methods in the form of \eqref{eq:gradient-descent} satisfy our definable assumption.

From a geometric viewpoint, definability ensures a form of rigidity. The following describes the main feature of definable functions used in \Cref{theorem:learning-rate-gradient-stability}. It states that, apart from a finite number of step sizes, being a smooth submanifold, is preserved by the inverse of the algorithmic recursion \Cref{eq:iterative-methods}, up to countable unions. The proof is postponed to \Cref{sec:semi-algebraic}. 

\begin{lemma}
    \label{lemma:gradient-update-good}
    Let $D \in \RR^{m \times m}$ be invertible, $g \colon \RR^m \to \RR^m$ be $C^1$ and \sa. Consider the function $G_\alpha$ defined as in \Cref{eq:iterative-methods}.
    There exists a subset $\Lambda \subseteq \RR_{> 0}$, whose complement is finite, such that for any $\alpha \in \Lambda$: if $S\subset \RR^m$ is a $C^1$ submanifold of dimension at most $m-1$ , the pre-image $G_\alpha^{-1}(S)$ is contained in a countable union of $C^1$ manifolds of dimension at most $m-1$.
\end{lemma}

\begin{remark}{\rm 
    \Cref{lemma:gradient-update-good} is used in the proof of \Cref{theorem:learning-rate-gradient-stability} to globalize the local stability result in \Cref{theorem:refined-center-stable-manifold}. \cite{mucsat2025gradient} independently provided a similar globalization strategy based on the abstract Luzin $N^{-1}$ property. Their main result is on gradient descent but it can be generalized to the abstract setting in \eqref{eq:iterative-methods}, leading to a similar, but slightly less precise characterization as in \Cref{theorem:learning-rate-gradient-stability}. }
\end{remark}

\subsection{Sketch of proof of \Cref{theorem:learning-rate-gradient-stability} and remarks}
\label{sec:proofSketch}
A fully detailed proof is given in \Cref{sec:semi-algebraic}, we only provide here a sketch and comments.
By \Cref{theorem:refined-center-stable-manifold}, locally around every unstable fixed point $x^\star \in \RR^n$ of the dynamical system:
\begin{equation*}
    x_{k + 1} = G_\alpha(x_k) := \mathbf{D} x_k - \alpha g(x_k),
\end{equation*}
there exists a stable center manifold of dimension strictly smaller than $n$. We deduce that if $x_{0} \in W_\alpha$, then there exists an index $k \in \NN$ such that $x_k$  lies in $\mathcal{S} \subseteq \RR^n$, a countable union of these submanifolds. In particular, $\mathcal{S}$ is \emph{small}, i.e., it has zero Lebesgue measure, and is equal to the countable union of submanifolds of dimension strictly smaller than $n$.

Next, by the previous reasoning, $W_\alpha$ has to be included in:
\begin{equation*}
    \bigcup_{k \in \NN} G_\alpha^{-k}(\mathcal{S}) \quad \text{where} \quad G_\alpha^{-k}(\mathcal{S}) := \underbrace{G_\alpha^{-1} \circ \ldots \circ G_\alpha^{-1}}_{k \text{ times}}(\mathcal{S}), 
\end{equation*}
the $k$ times pre-image of the mapping $G_\alpha$. To conclude that $W_\alpha$ is also small, we need $G_\alpha^{-1}$ to preserve small sets, i.e., if $\mathcal{S}$ is small, then so is $G_\alpha^{-1}(\mathcal{S})$.

However, this is generally not true for $G_\alpha$ (e.g., $G_\alpha$ is a constant map). Existing works used technical assumptions to avoid this problem. Here is a (non)-exhaustive list:
\begin{enumerate}[leftmargin=*]
    \item In \cite{lee2016gradient,panageas2017gradient,sun2019heavy,o2019behavior}, the authors assume that the function has gradient Lipschitz. If this is the case, by choosing $\alpha$ small enough (e.g., $\alpha < \frac{1}{L}$ for gradient descent), one can prove that $G_\alpha$ is a diffeomorphism. Therefore, $G_\alpha^{-1}:\RR^n \to \RR^n$ is also a diffeomorphism, and the result follows.
    \item In \cite{ahn2022understanding}, the authors assume explicitly that $G_\alpha^{-1}$ preserves small sets.
    \item In \cite{cohen2021gradient,o2019behavior}, the authors work with quadratic objective functions. Therefore, the form of $G_\alpha$ is  simpler 
to analyze. 
\end{enumerate}
These illustrate the fact that going from \Cref{theorem:refined-center-stable-manifold} to a result similar to \Cref{theorem:learning-rate-gradient-stability} is not direct.
Our approach is to use \Cref{lemma:gradient-update-good}, ensuring that the set of $\alpha$ such that $G_\alpha^{-1}$ does not preserve small sets is at most \emph{finite}. This allows to remove all technical assumptions tailored for each specific algorithm outlined above, and unify them into a single framework.
Compared to \cite{ahn2022understanding}, we provide a qualitative description set of step sizes $\alpha$ such that $G_\alpha^{-1}$ preserves small sets. Compared to \cite{lee2016gradient,panageas2017gradient,sun2019heavy,o2019behavior}, we remove the small step size assumption.

\section{Applications to optimization algorithms and eigenvalue filtering}
\label{sec:applications}
In the following, we apply \Cref{theorem:learning-rate-gradient-stability} to multiple optimization algorithms to quantify the {\em eigenvalue filtering} phenomenon. We use the ``almost every'' formulation of the introduction for simplicity. Technical details are postponed to \Cref{sec:algorithms}.

\subsection{Cauchy's gradient descent}
The update rule of gradient descent (GD) is given by:
\begin{equation}
    \label{eq:gradient-descent}
    x_{k+1} = x_k - \alpha \nabla f(x_k).
\end{equation}

\begin{proposition}[Gradient descent eigenvalue filtering]
    \label{cor:application-gradient-descent}
    Assume that $f$ is $C^2$ and \sa. For almost every $\alpha > 0$ and $x_0 \in \RR^n$, we have: if $\{x_k\}_{k \in \NN}$ given by \Cref{eq:gradient-descent} converges to $\bar{x}$, then all eigenvalues $\lambda$ of $\nabla^2 f(\bar{x})$ satisfy: $0 \leq \lambda \leq \frac{2}{\alpha}$.
\end{proposition}
The stability condition is very well known,
let us compare \Cref{cor:application-gradient-descent} with existing results.

\emph{Avoidance of strict saddle points:} The constraint $\lambda \geq 0$ illustrates the well known fact that gradient descent escapes strict saddle points. This result has already been stated multiple times in different forms (e.g., \cite{thom1949partition,goudou2009gradient,lee2016gradient,panageas2017gradient}). In particular, \cite[Theorem 1]{pemantle1990nonconvergence} applies to stochastic gradient descent with vanishing step-sizes, \cite[Theorem 4.1]{lee2016gradient} and \cite[Theorems 2,3]{panageas2017gradient} applies to Lipschitz gradients in small step regime.
In comparison \Cref{cor:application-gradient-descent} requires minimal qualitative assumptions on $f$ and is valid for a much broader range of step sizes.
    
\emph{Large step sizes and small curvature:} the upper-bound $\lambda \leq 2/\alpha$ is a core element of the Edge Of Stability (EOS) phenomenon \cite{cohen2021gradient}, crucial for understanding training dynamics. Most often in the EOS literature, stability mechanisms are justified on quadratic objectives, for which computation is very simple. Our result shows that these conclusions extend to a generic deep learning setting.  \Cref{cor:application-gradient-descent} is similar to \cite[Theorem 1]{ahn2022understanding} without the need for the abstract \cite[Assumption 1]{ahn2022understanding}, for a generic step-size.

\subsection{Polyak's Heavy Ball method}
The method's iterations update is given by:
\begin{equation}
    \label{eq:momentum-gradient}
    \begin{pmatrix}
        x_{k+1} \\y_{k+1}
    \end{pmatrix} = 
    \underbrace{\begin{pmatrix}
        (1 + \beta) I & -\beta I \\
        I & 0_{n \times n}
    \end{pmatrix}}_{D} \begin{pmatrix}
        x_k \\ y_k
    \end{pmatrix} - \alpha \underbrace{\begin{pmatrix}
        \nabla f(x_k) \\ {0}
    \end{pmatrix}}_{g(\cdot)},
\end{equation}
where $I$ is the identity matrix and $0_{n \times n}$ is an all-zero one.
Considering \Cref{theorem:learning-rate-gradient-stability}, the eigenvalue computation is also known and was carried out for example in \cite{polyak1987introduction}.
\begin{proposition}[Heavy Ball eigenvalue filtering]
    \label{cor:application-momentum-gradient}
    Assume that $f$ is $C^2$, \sa and $0 < \beta < 1$ in \Cref{eq:momentum-gradient}. For almost every $\alpha > 0$ and $(x_0,y_0) \in \RR^n \times \RR^n$, we have: if $\{(x_k, y_k)\}_{k \in \NN}$ converges to some $(\bar{x}, \bar{y})$, then all the eigenvalues $\lambda$ of $\nabla^2 f(\bar{x})$ satisfy:
    \begin{equation*}
        0 \leq \lambda \leq \frac{2(1 + \beta)}{\alpha}.
    \end{equation*}
\end{proposition}

\emph{Avoidance of strict saddle points and generic convergence to minimizers:}  for the heavy ball method, this was documented for the continuous time ODE limit of the method \cite{goudou2009gradient}, and in discrete time for Lipschitz gradients under small step size conditions \cite{sun2019heavy,castera2021InertialNA}. Our result holds for generic step-sizes.

\emph{Large step sizes and small curvature:} if the iterates of \Cref{eq:momentum-gradient} converges, the limiting curvature is upper bounded by $2(1 + \beta)/\alpha$. This upper bound appeared in \cite[Equation 1, Theorem 2]{cohen2021gradient} for the quadratic case and we extend it to a general setting.

\subsection{Nesterov's accelerated gradient method}
We consider Nesterov's accelerated gradient method (NAG) with fixed momentum, the original version being described with decaying momentum \cite{nesterov1983method}. Fixed momentum is frequently used in deep learning, for example in the Pytorch implementation. Its iteration update is given by:
\begin{equation}
    \label{eq:nesterov-accelerated-gradient}
    \begin{pmatrix}
        x_{k+1} \\y_{k+1}
    \end{pmatrix} = 
    \underbrace{\begin{pmatrix}
        (1 + \beta) I & -\beta I \\
        I & {0}_{n \times n}
    \end{pmatrix}}_{D} \begin{pmatrix}
        x_k \\ y_k
    \end{pmatrix} - \alpha \underbrace{\begin{pmatrix}
        \nabla f(x_k + \beta(x_k - y_k)) \\ {0}
    \end{pmatrix}}_{g(\cdot)}
\end{equation}

\begin{proposition}[``Nesterov's accelerated gradient"\footnote{Here we adopt the ML community's terminology: `accelerated' refers to the ideal case when the loss is strongly convex.} eigenvalue filtering]
    \label{cor:application-nesterov-gradient}
    Assume that $f$ is $C^2$, \sa and $0 < \beta < 1$ in \Cref{eq:nesterov-accelerated-gradient}. For almost every $\alpha > 0$ and $(x_0,y_0) \in \RR^n \times \RR^n$, we have: if $\{(x_k, y_k)\}_{k \in \NN}$ converges to some $(\bar{x}, \bar{y})$ , then all the eigenvalues $\lambda$ of $\nabla^2 f(\bar{x})$ satisfy:
    \begin{equation*}
        0 \leq \lambda \leq \frac{1}{\alpha} \left( \frac{2 + 2\beta}{1 + 2\beta}\right).
    \end{equation*}
\end{proposition}
An existing line of works \cite{jin2018accelerated,agrawal2017finding,carmon2018accelerated} exploits the ideas of Nesterov's method to investigate the complexity of finding approximate second order critical points. The nature of our results is different: the method generically avoids strict saddle points and filters eigenvalues more strongly as $\beta$ grows. In addition, the upper bound on eigenvalues was described for the quadratic case \cite[Appendix B]{cohen2021gradient}.\\

\emph{Generic convergence and eigenvalue filtering}: 
This algorithm shares the property of generic convergence to minimizers with the gradient descent~(\ref{eq:gradient-descent}), and the Heavy Ball method~(\ref{eq:momentum-gradient}). Observe that its  filtering abilities are slightly  improved over the gradient descent ($2/\alpha$) and the Heavy Ball method ($4/\alpha$), as the eigenvalue upper bound tends to $4/(3\alpha)$ when $\beta$ approaches $1$.

\subsection{Unnormalized Sharpness Aware Minimization (USAM)}
USAM was introduced and studied in \cite{andriushchenko22towards}. Its iteration update is given by:
\begin{equation}
    \label{eq:unnormalized-sam}
    x_{k+1} = x_k - \alpha \nabla f(x_k + \rho \nabla f(x_k))
\end{equation}
for some constant $\rho > 0$. This is a modified version of the original (normalized) SAM \cite{foret2021sharpnessaware}, given by: $x_{k+1} = x_k - \alpha \nabla f\left(x_k + \rho \frac{\nabla f(x_k)}{\|\nabla f(x_k)\|}\right)$.

We focus on \Cref{eq:unnormalized-sam} because the update rule of normalized SAM is not $C^1$, it is actually discontinuous around critical points. In practice, this is combined with other techniques such as momentum:
$x_{k+1} = x_k + \beta (x_k - x_{k-1}) - \alpha \nabla f(x_k + \rho \nabla f(x_k))$,
or equivalently,
\begin{equation}
    \label{eq:unnormalized-momentum-sam-standard-form}
    \begin{pmatrix}
        x_{k+1} \\y_{k+1}
    \end{pmatrix} = 
    \underbrace{\begin{pmatrix}
        (1 + \beta) I & -\beta I \\
        I & {0}_{n \times n}
    \end{pmatrix}}_{D} \begin{pmatrix}
        x_k \\ y_k
    \end{pmatrix} - \alpha \underbrace{\begin{pmatrix}
        \nabla f(x_k + \rho\nabla f(x_k)) \\ \mathbf{0}
    \end{pmatrix}}_{g(\cdot)}
\end{equation}

\begin{proposition}[USAM + Heavy Ball momentum eigenvalue filtering]
    \label{cor:application-usam-momentum}
    Assume that $f$ is $C^2$, \sa and $0 \leq \beta < 1$ in \Cref{eq:unnormalized-momentum-sam-standard-form}. For almost every $\alpha > 0$ and $(x_0,y_0) \in \RR^n \times \RR^n$, we have: if $\{(x_k, y_k)\}_{k \in \NN}$ converges to some $(\bar{x}, \bar{y})$ where $\nabla f(\bar{x}) = 0$, then all the eigenvalues $\lambda$ of $\nabla^2 f(\bar{x})$ satisfy: $0 \leq \lambda(1 + \rho \lambda) \leq \frac{2(1 + \beta)}{\alpha}$,
    or equivalently,
    \begin{equation*}
        \frac{-1 - \sqrt{1 + 8(1+\beta)\rho/\alpha}}{2\rho} \leq \lambda \leq -\frac{1}{\rho} \quad \text{or} \quad 0 \leq \lambda \leq \frac{\sqrt{1 + 8(1 + \beta)\rho/\alpha}-1}{2\rho}.
    \end{equation*}
\end{proposition}


\emph{Strict saddle points may be attractive}: USAM with or without momentum does not avoid strict saddle points generically. For example, if $f(x) = -\frac{1}{\rho} x^2$ and $0 < \alpha < \rho$, then $(0,0)$ is a stable fixed point since one can prove that the update in \Cref{eq:unnormalized-momentum-sam-standard-form} is locally a contraction at $(0,0)$. This was remarked in \cite[Theorem 1]{kim2023stability} for the ODE version of \Cref{eq:unnormalized-sam} and is seen in \Cref{cor:application-usam-momentum} with the negative interval.

 \emph{Apparition of new fixed points}: The fixed points of USAM are given by
\[
\{(x,x):\ x + \rho\nabla f(x)\in \crit f\}.
\]
Thus, fixed points of USAM is possibly \emph{strictly} larger than the set of critical points of the underlying loss function. As shown in \Cref{appearance}, the set of fixed points of the USAM algorithm not belonging to $\crit f$ may even have a nonempty interior. 
This remark, combined with the previous one on strict saddle points, illustrates that, in full generality, the USAM algorithm does not enjoy the property of generic convergence to local minimizers of the objective $f$, contrary to the gradient or heavy ball algorithms. {This is, however, a worst-case statement, at the level of generality of our assumptions ($C^2$, definable objectives): the absence of guarantee does not imply a practical impact, since the objective functions encountered in practice may exhibit additional properties (e.g. quadratic growth near the initialization, or a favorable landscape)}. In the context of deep learning, we observe however that the USAM algorithm has a minimizing behavior, suggesting that the spurious fixed points have a limited impact.

\emph{Eigenvalue filtering}: The upper bound in \Cref{cor:application-usam-momentum} is smaller than that of \Cref{cor:application-gradient-descent} for any $\rho > 0$. Both the upper and lower bounds are of order $1/ \sqrt{\alpha \rho}$ as $\rho \to \infty$. The upper bound in \Cref{cor:application-usam-momentum} was described in \cite{zhou2025sam} for a simplified version of USAM. Moreover, for a fixed $\rho>0$, as $\alpha \to 0$, these bounds scale like $O(1/\sqrt{\alpha})$, while for previous methods the upper bound scales like $O(1/\alpha)$. 
{In other words, the stability bound of USAM excludes a larger set of limiting critical points, on the basis of their local curvature, than that of gradient descent.}

\subsection{Two variants of USAM {with stricter stability bounds}}
We investigate two variations on USAM which result in finer constraints on asymptotic curvature. 
For both, with a fixed SAM parameter $\rho$, the upper bound scales like $\alpha^{-1/3}$ as $\alpha \to 0$, which is smaller the one found for USAM. {These two algorithms are introduced to illustrate \Cref{theorem:learning-rate-gradient-stability} in light of our main result, if one runs them with a given $\rho > 0$, the upper bounds on the sharpness of the limit point are smaller than for the base algorithms (gradient descent or the heavy ball method). We do not claim any computational or generalization advantage over SAM or USAM, nor that they should be used in practice since each iteration is more expensive than an iteration of USAM (an additional gradient evaluation for Two-step USAM, a Hessian-vector product for Hessian USAM).}

\subsubsection{Two-step USAM}
The following update performs two gradient ascent steps:
\begin{equation}
    \label{eq:unnormalized-sam-sam}
    x_{k+1} = x_k - \alpha \nabla f(\underbrace{x_k + \rho \nabla f(x_k) + \rho \nabla f(x_k + \rho \nabla f(x_k))}_{\text{two gradient ascent steps}})
\end{equation}
for some constant $\rho > 0$. Combining with heavy ball momentum gives:
\begin{equation}
    \label{eq:momentum-usam}
    \begin{pmatrix}
        x_{k+1} \\y_{k+1}
    \end{pmatrix} = 
    {\begin{pmatrix}
        (1 + \beta) I & -\beta I \\
        I & {0}_{n \times n}
    \end{pmatrix}} \begin{pmatrix}
        x_k \\ y_k
    \end{pmatrix} - \alpha \begin{pmatrix}
        \nabla f(x_k + \rho \nabla f(x_k) + \rho \nabla f(x_k + \rho \nabla f(x_k))) \\ \mathbf{0}
    \end{pmatrix} 
\end{equation}

\begin{proposition}[Two-step USAM gradient eigenvalue filtering]
    \label{cor:application-usam-momentum-two}
    Assume that $f$ is $C^2$, \sa, $\rho > 0$ and $0 \leq \beta < 1$ in \Cref{eq:momentum-usam}. For almost every $\alpha > 0$ and $(x_0,y_0) \in \RR^n \times \RR^n$, we have: if $\{(x_k,y_k)\}_{k \in \NN}$ converges to $(\bar{x}, \bar{y})$ where $\bar{x}$ is a critical point of $f$, i.e., $\nabla f(\bar x)=0$, then all the eigenvalues of $\lambda$ of $\nabla^2 f(\bar{x})$ satisfy:
    \begin{equation*}
        0 \leq \lambda (1 + \rho \lambda)^2 \leq \frac{2(1 + \beta)}{\alpha}.
    \end{equation*}
\end{proposition}

\begin{remark}[Convergence \& improved eigenvalue filtering]
  {\rm  \label{rem:two-step-usam-vs-usam}
   {Contrary to USAM, Two-step USAM excludes limiting critical points having a negative
    Hessian eigenvalue, with the exception of the boundary value $\lambda = -1/\rho$:
    indeed, $\lambda(1+\rho\lambda)^2 \ge 0$ if and only if $\lambda \ge 0$ or
    $\lambda = -1/\rho$. Hence, when $\nabla^2 f$ has an eigenvalue equal
    to $-1/\rho$ at some critical point, generic avoidance of strict saddle points may
    fail, see also the example in \Cref{sec:spuriousExampleTwoStepUSAM}. Away from this edge case, the stability bound of
    \Cref{cor:application-usam-momentum-two} excludes all limiting critical points with a negative
    Hessian eigenvalue.}
   Yet the nonempty interior argument of \Cref{appearance} maybe adapted and, for simple costs, the algorithm may have many spurious fixed points, which do not correspond to critical points of the objective. As for USAM, empirical results suggest that they have a limited effect on the minimizing behavior of the algorithm in deep learning.
   As for eigenvalue filtering, the result is rather positive; if USAM and two-step USAM have the same common hyperparameters, we infer that, within $\crit f$, stable fixed points for Two-step USAM are also stable for USAM (the converse being not necessarily true). This suggests that Two-step USAM {excludes a larger set of limiting critical points, on the basis of their local curvature, than USAM}.}
\end{remark}

\subsubsection{Hessian USAM}

We consider the following iteration update:
\begin{equation}
    \label{eq:unnormalized-hsam}
    x_{k+1} = x_k - \alpha \nabla f(x_k + \rho \nabla^2 f(x_k) \nabla f(x_k)),
\end{equation}
replacing $\nabla f(x_k)$ as in \Cref{eq:gradient-descent} by $\nabla f(x_k + \rho \nabla^2 f(x_k)\nabla f(x_k))$. {Since this update involves $\nabla^2 f$, we assume throughout this section that $f$ is $C^3$, so that the mapping $g$ in \Cref{eq:iterative-methods} is $C^1$ as required by \Cref{theorem:learning-rate-gradient-stability}.} Combining with heavy ball momentum gives:
\begin{equation}
    \label{eq:momentum-hsam}
    \begin{pmatrix}
        x_{k+1} \\y_{k+1}
    \end{pmatrix} = 
    {\begin{pmatrix}
        (1 + \beta) I & -\beta I \\
        I & {0}_{n \times n}
    \end{pmatrix}} \begin{pmatrix}
        x_k \\ y_k
    \end{pmatrix} - \alpha {\begin{pmatrix}
        {\nabla f(x_k + \rho \nabla^2 f(x_k) \nabla f(x_k))} \\ \mathbf{0}
    \end{pmatrix}}
\end{equation}

\begin{proposition}[Hessian USAM gradient eigenvalue filtering]
    \label{cor:application-hsam}
    Assume that $f$ is $C^{{3}}$, \sa, $\rho > 0$ and $0 \leq \beta < 1$ in \Cref{eq:momentum-hsam}. For almost every $\alpha > 0$ and $(x_0,y_0) \in \RR^n \times \RR^n$, we have: if $\{(x_k,y_k)\}_{k \in \NN}$ converges to $(\bar{x}, \bar{y})$ where $\bar{x}$ is a critical point of $f$, i.e., $\nabla f(\bar x)=0$, then all the eigenvalues of $\lambda$ of $\nabla^2 f(\bar{x})$ satisfy:
    \begin{equation*}
        0 \leq \lambda (1 + \rho \lambda^2) \leq \frac{2(1+\beta)}{\alpha}.
    \end{equation*}
\end{proposition}

\begin{remark}[Convergence \& improved eigenvalue filtering]
    \label{rem:hessian-usam-vs-usam}
{\rm  Hessian USAM avoids strict saddle points {since $\lambda (1 + \rho \lambda^2) \geq 0$ implies $\lambda \geq 0$. There is no boundary case $\lambda = - \frac{1}{\rho}$ as in the previous method}. However, it may also fail to achieve generic convergence to local minimizers because spurious fixed points may generate stable points out of $\crit f$, recall \Cref{appearance}. Once again, its eigenvalue-filtering properties are generally improved over USAM.}\end{remark}

\section{Experiments}

\begin{figure}
    \centering
    \includegraphics[width=0.45\linewidth]{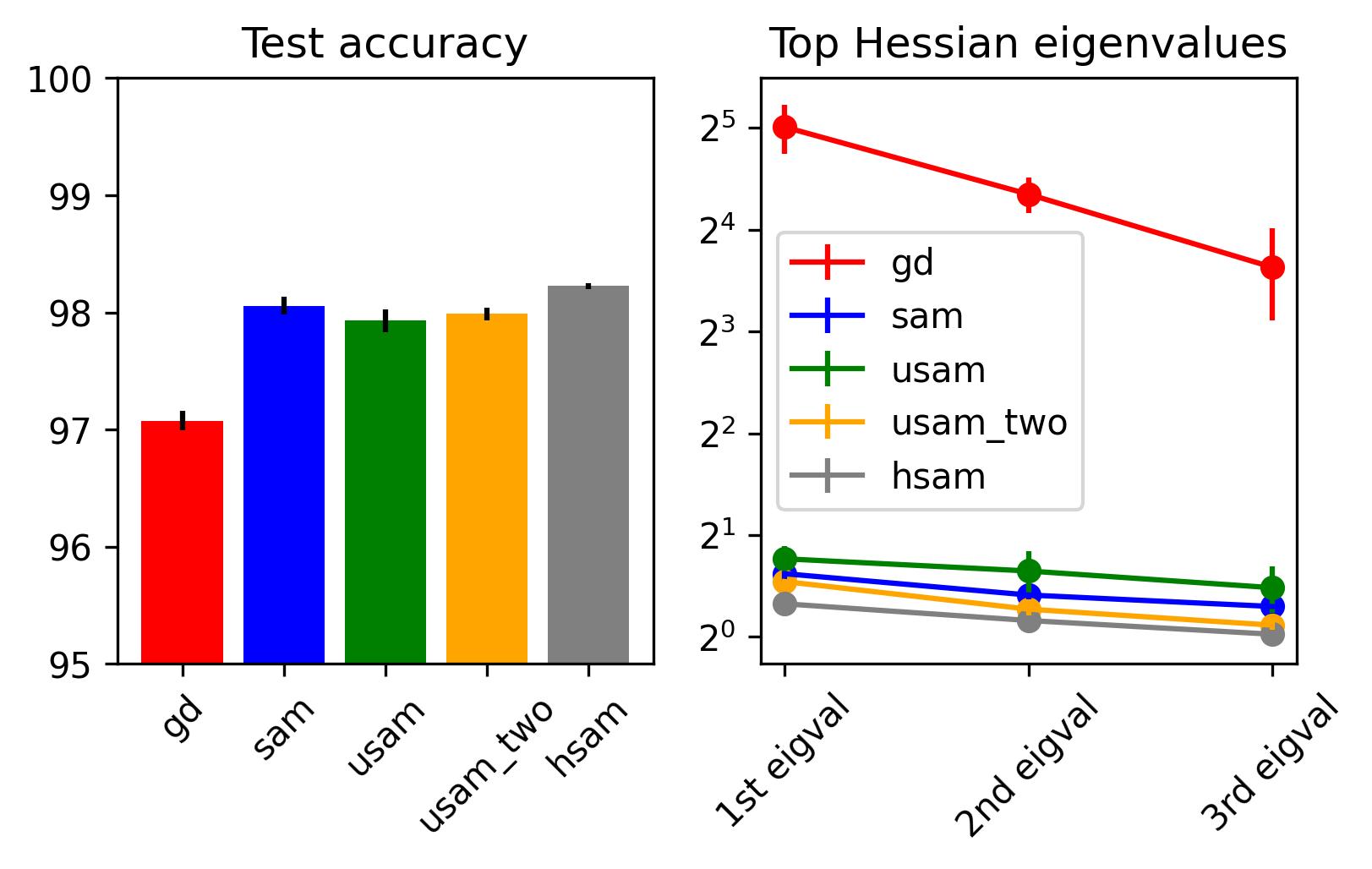} \quad \includegraphics[width=0.45\linewidth]{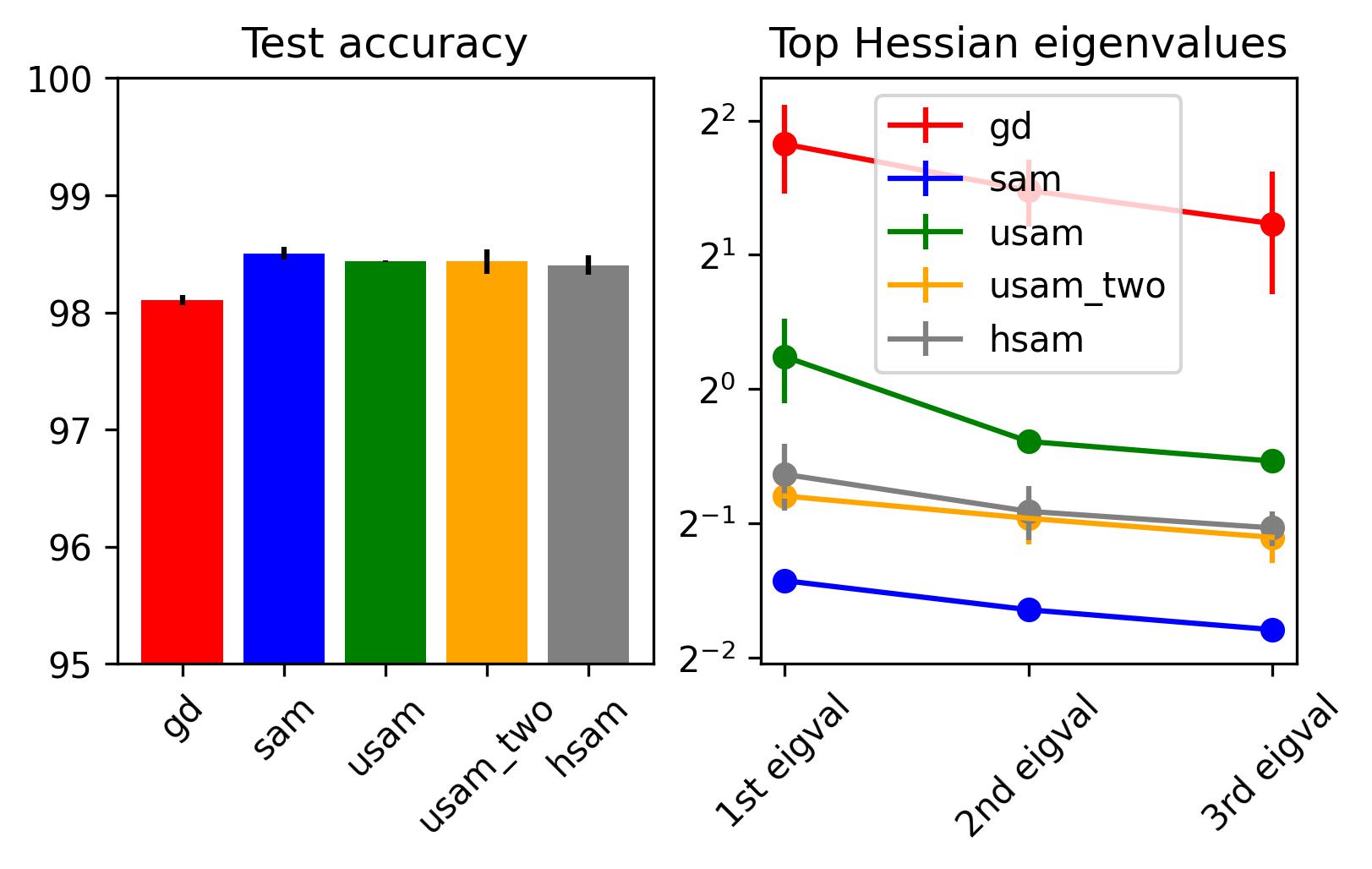}
    \caption{(\textbf{Experiment 1}) - MLP trained on MNIST with \emph{stochastic gradient descent} and its corresponding to SAM, USAM, USAM2 and Hessian USAM versions. Left without momentum, right with $\beta = 0.9$. SAM, USAM and USAM2 are trained with $\rho \in \{0.05, 0.1, 0.2\}$ while Hessian USAM is trained with $\rho \in \{0.01, 0.02, 0.05, 0.1, 0.2\}$. Among these $\rho$, we choose those yielding the best models (in terms of test accuracy) and report their accuracy and Hessian spectra.}
    \label{fig:mnist-sgd}
\end{figure}

In this section, we evaluate numerically limiting curvature at equilibrium for the considered algorithms, in the context of neural networks training. In our experiments, we compare in particular the popular (stochastic) gradient descent and heavy-ball method with their corresponding USAM, Two-step USAM, and HSAM versions to observe the effect of eigenvalue filtering. {The goal of these experiments is to compare the algorithms, for various values of the hyperparameters, in order to test the effect of eigenvalue filtering on the curvature of the obtained solutions; they are only illustrative and are not meant to demonstrate a practical advantage of one method over another. In particular we claim no practical advantage for Two-step USAM and Hessian USAM.}
We do not implement the Nesterov algorithm and its SAM variants since they are less commonly used in the neural networks training context.

\subsection{Neural network training experiments}
We conducted three neural network training experiments described below; our Python implementation is available at \url{https://github.com/tung-qle/uncle_sam} for reproduction purposes. The datasets, architectures and protocols are as follows:
\begin{enumerate}[leftmargin=*]
    \item \textbf{MNIST dataset and MultiLayer Perceptron (MLP)}: The dimensions of hidden layers are $\{128, 64, 10, 10\}$, with ReLU activation function. We use the standard cross-entropy loss for classification.
    
    We consider GD, SAM, USAM, USAM2 and Hessian SAM with ($\beta = 0.9$) and without $(\beta = 0$) momentum. We fix a $128$ minibatch size, a learning rate $\alpha = 0.01$ and a weight decay of $5e-4$. The parameter $\rho$ is tuned from grid search: for SAM variants, we tune the hyperparameter $\rho \in \{10^{-i}, 2 \times 10^{-i}, 5 \times 10^{-i} \mid i \geq 1, i \in \NN\}$. We consider the best run in terms of accuracy among the three largest values of $\rho$ such that the training does not fail. We report the corresponding three largest eigenvalues of the Hessian matrix of the training loss after training. The training is repeated three times for each set of hyperparameters. The results are illustrated in \Cref{fig:mnist-sgd}.

    \item \textbf{MNIST-fashion dataset and MultiLayer Perceptron (MLP)}: We use the same setting as in the first experiment, except replacing the MNIST dataset with the MNIST-fashion dataset \cite{xiao2017fashion}. The results are illustrated in \Cref{fig:mnist-fashion-sgd} in \Cref{sec:additionalResults}.

    \item \textbf{CIFAR10 dataset and WideResNet-16-8}: We train a WideResNet-16-8 (\cite{zagoruyko2016WRN}) without batch normalization layers with CIFAR10 dataset. This specification echoes the remark from \cite[Section 4.2]{foret2021sharpnessaware} whose authors suggest that batch normalization layers tend to ``obscure interpretation of the Hessian''. Our choice of WideResNet architecture is motivated by previous experiments reporting successful training without batch normalization. We consider the momentum version of GD, SAM, USAM, USAM2 and Hessian SAM with the same value of $\rho = 0.001$. The results are shown in \Cref{fig:wideresnet-cifar10}.
\end{enumerate}


\begin{figure}
    \centering
    \includegraphics[width=0.6\linewidth]{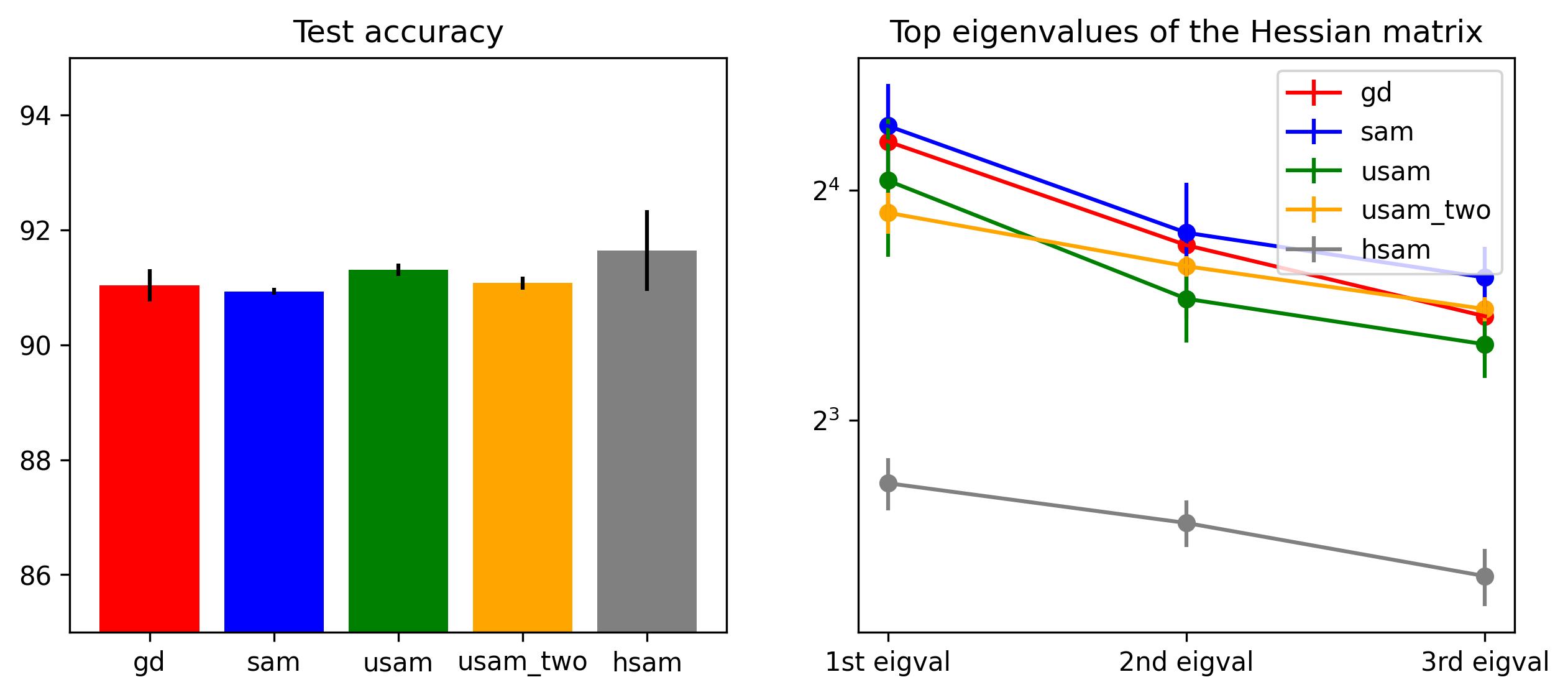}
    \caption{(\textbf{Experiment 3}) - Models are trained with \emph{stochastic gradient descent} and its corresponding to SAM, USAM, USAM2 and Hessian USAM versions. The values of $\rho$ of all SAM-like algorithms are set at $\rho = 0.001$.}
    \label{fig:wideresnet-cifar10}
\end{figure}

The MLP experiments illustrate the fact that our theoretical findings transfer well to the experimental setting: in general, methods such as USAM, USAM2 and Hessian USAM consistently find flatter (or low-curvatured) minimizers in comparison to their vanilla versions. The experiment is far from the idealized assumptions in \Cref{theorem:learning-rate-gradient-stability}, with non-smoothness (since we use ReLU neural networks) and stochasticity (in the optimization algorithms), but the theory definitely aligns with empirical results: USAM2 and Hessian USAM filter Hessian eigenvalues more efficiently, and find therefore solutions with wider basins.

{The comparison between the left and right panels of \Cref{fig:mnist-sgd} also reveals a limitation of our analysis. For a fixed step size $\alpha$, the upper bounds of all results increase with the momentum parameter $\beta$ (they involve the factor $1 + \beta$)), so that our theory predicts a weaker eigenvalue filtering with momentum, whereas the top Hessian eigenvalues measured at the end of training are smaller with $\beta = 0.9$ than without momentum. Our theory and these experiments therefore tell a different story regarding the effect of momentum, and we leave the resolution of this discrepancy open for future work. A possible explanation for this mismatch is that our theory does not account for stochastic noise: in the stochastic setting, momentum not only has a dynamical effect, it also has a variance reduction effect, as it implements a running average of the stochastic gradients, which cannot be captured by our deterministic analysis.}

As for the WideResNet the situation is not as clear in \Cref{fig:wideresnet-cifar10}, the differences are less pronounced, notably between USAM and USAM2. This architecture is much more difficult to train than the MLP architecture considered above. Another outcome of the experiment, which aligns well with our stability analysis is as follows. For a fixed $\alpha$, USAM2 and Hessian USAM require much smaller values of $\rho$ than USAM to avoid training failure. This is consistent with the fact that both USAM2 and Hessian USAM enforce more restrictions on the limiting curvature in comparison to USAM. For this reason we needed to significantly decrease the value to $\rho = 0.001$ in order to obtain successful training. In this regime, the reduction of asymptotic curvature is limited.

Note that SAM also consistently finds flat minimizers, this fact is experimentally confirmed by previous works \cite{foret2021sharpnessaware,tan2024stabilizing} in several deep learning settings. Our empirical result resonates with these observations. Nevertheless, our theoretical results do not provide any explanation for this behavior and we leave this extension to future work.

Our last remark is that with architectures using batch normalization, SAM (and also USAM, USAM2 and Hessian SAM) does not empirically converge to flatter minimizers (see \cite[Section 4.2]{foret2021sharpnessaware}). This does not contradict our theoretical result because the presence of batch normalization changes the dynamics of the algorithm. Another future direction is to extend \Cref{theorem:learning-rate-gradient-stability} to also cover batch normalization operations.

\subsection{On the predictive power of the bound}
Our main result is about deterministic algorithms, without subsampling. We report experiments with subsampling because this is closer to practical implementations. To verify the tightness of the bound, we conduct similar experiments on deterministic full-batch variants.

For each algorithm, with its best-tuned $\rho$, we compare the obtained spectral radius and its upper bounds. The results and experimental details are presented in \Cref{appendix:extra-experiments}. We observe that the orders of magnitude align well. This illustrates the tightness of the bound and shows empirically that it is indeed predictive of what happens in practice. In this particular setting, USAM yields the smallest sharpness, because it actually converges with a much higher value $\rho$ than HSAM and Two-step USAM. 

\subsection{On the parameter $\rho$, hyperparameters, flatness and curvature}

From the theoretical viewpoint, $\rho$ is a hyperparameter, as $\alpha$ and $\beta$. It will influence the final upper bounds of the spectral radius. As a conclusion, if one can run USAM with a much higher value $\rho$ than HSAM or Two-step USAM, then the order of upper bounds will be changed. That is why we first illustrate our eigenvalue filtering result with the same $\rho$ for all methods.

In practice, one needs to generally tune $\rho$ for each specific algorithm. In our experiments, $\rho \in \{0.05, 0.1, 0.2\}$ for USAM and Two-step USAM, while $\rho \in \{0.01, 0.02, 0.05, 0.1, 0.2\}$ for HSAM. In our experiments, HSAM needs a smaller $\rho$ to be able to converge, in comparison to USAM and Two-step USAM. The full-batch experiments are reported with the best $\rho$ for each algorithm.

Finally, similarly to $\rho$, further hyperparameters also exert an impact on the flatness of the obtained solution. By computing the eigenvalues of the Jacobian of $G_\alpha$ (that might depend on other hyperparameters), we can anticipate whether varying these hyperparameters will eventually increase or decrease the upper bounds, and thus, the actual curvature.

\section{Conclusion}
We provide a simple, general, and versatile theoretical result on eigenvalue filtering (cf. \Cref{theorem:learning-rate-gradient-stability}) in the context of convergence of optimization algorithms. This takes the form of a variation on the Hadamard--Perron stable manifold theorem, which simplifies and generalizes existing results of this type in the machine learning literature. The proposed result aligns with recent empirical and theoretical advances in sharpness-aware minimization, large step size, generalization, and edge-of-stability phenomena.

We introduced two new algorithms, Two-step USAM and Hessian SAM. These algorithms are given to illustrate our theoretical findings on algorithmic stability (\Cref{theorem:learning-rate-gradient-stability}) in a deep network training scenario. We emphasize that the computational cost of a single iteration for each algorithm is higher than that of USAM. For this reason, we do not have empirical evidence that Two-step USAM or Hessian SAM provides a substantially practical advantage compared to SAM or USAM. They nonetheless illustrate the generality of the proposed theoretical analysis, and we leave extensive benchmarking of their empirical performance to future work. {Pushing the theory beyond this smooth, deterministic setting, which is still not settled (\Cref{tab:summary-results}), is a major challenge which we leave for future work.}


\subsubsection*{Broader Impact Statement}
This paper shows that the choice of hyperparameters of optimizers can have a filtering effect on the eigenvalues of the Hessian of fixed points. We do not see any potential negative impacts worth discussing.

\subsubsection*{Acknowledgements}
{
The authors thank the anonymous reviewers for careful reading and suggestions which helped improve our work.
This work has been funded by MIAI Cluster (ANR-19-P3IA-0003, ANR-23-IACL-0006). It also benefited from the support of the FMJH Program Gaspard Monge for optimization and operations research and their interactions with data science.
EP and JB thank TSE-P and acknowledges the support of the AI Interdisciplinary Institute ANITI
funding, through the ANR under the France 2030 program (grant ANR-23-IACL-0002), Chair TRIAL, ANR
MAD and ANR Regulia. 
As part of the Madlearning project, this work received funding from the French government, managed by the National Research Agency (ANR) under the France 2030 program, with reference number `ANR-25-PEIA-0002'.
}

\bibliography{references}
\bibliographystyle{tmlr}

\appendix
\section{Additional results and comments from the main text}
\label{sec:additionalResults}
\begin{example}
    \label{ex:genericityRequired}
    Let $h \colon \RR \to \RR$ be $C^2$, such that $h(t) = (t^2 - 1)^2$ if $t\leq 2$ and $h(t) = t^2$ if $t\geq 3$, and consider $f \colon \RR^m \to \RR$ such that $f(x) = h(\|x\|)$. 
    The origin is a strict local maximizer such that $\nabla^2 f(0) = - 4 I$ ($I$ is the identity matrix with proper sizes). For any $\alpha>0$, $\nabla f(0) = 0$ so that $0$ is fixed point of the gradient recursion with $\rho( \jac G_{\alpha}(0)) = 1 + 4 \alpha >1$, hence $ 0 \in W_\alpha$.
    Furthermore for $\alpha = \frac{1}{2}$, $x - \alpha \nabla f(x) = 0$ for any $x$ such that $\|x\|\geq 3$. Hence $\{x \in \RR^n \mid \|x\|\geq 3\} \subset W_\alpha$ and $W_\alpha$ is not as in \Cref{theorem:learning-rate-gradient-stability}.
\end{example}

\begin{figure}[h!]
    \centering
    \includegraphics[width=0.45\linewidth]{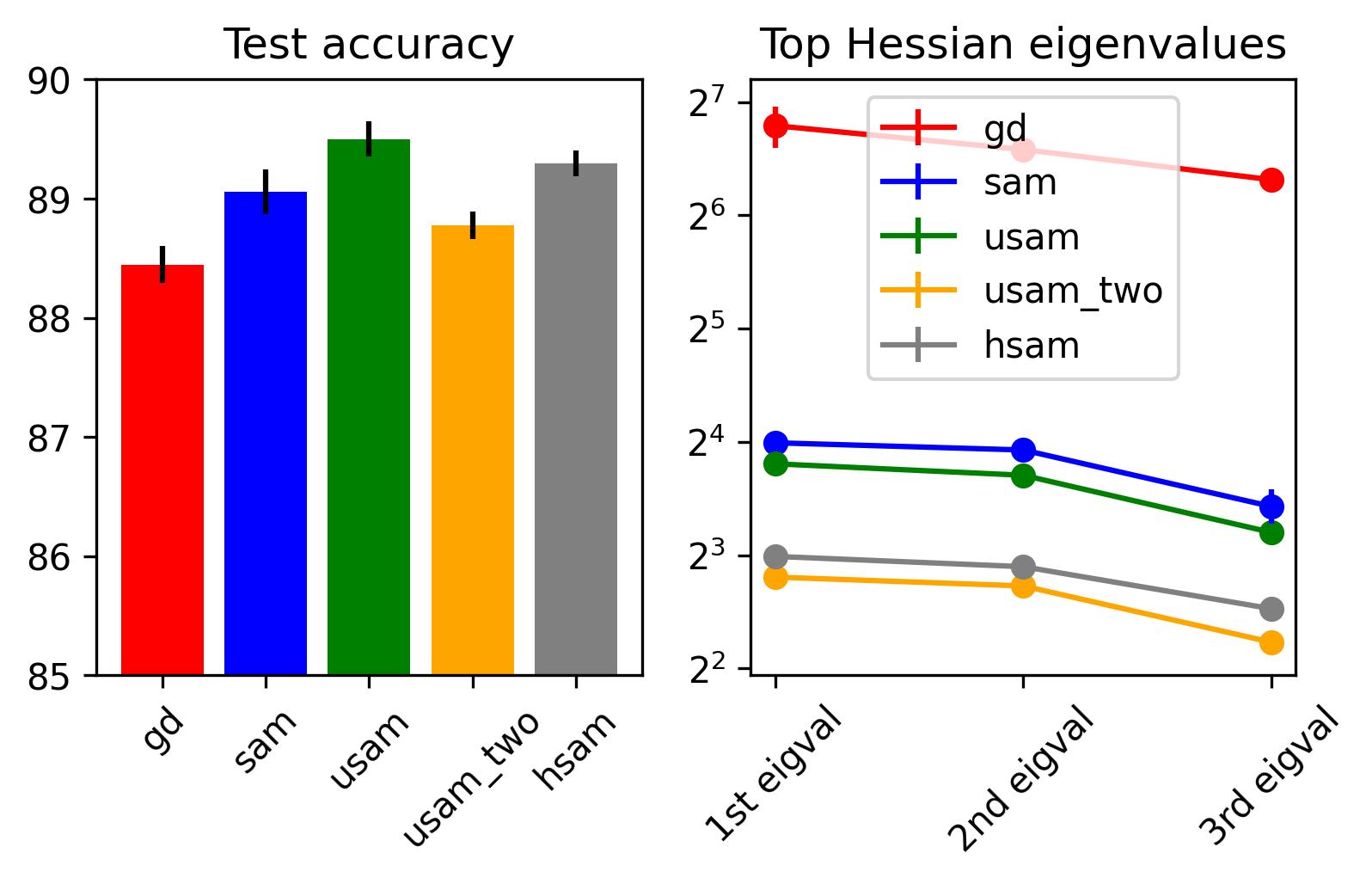} \quad \includegraphics[width=0.45\linewidth]{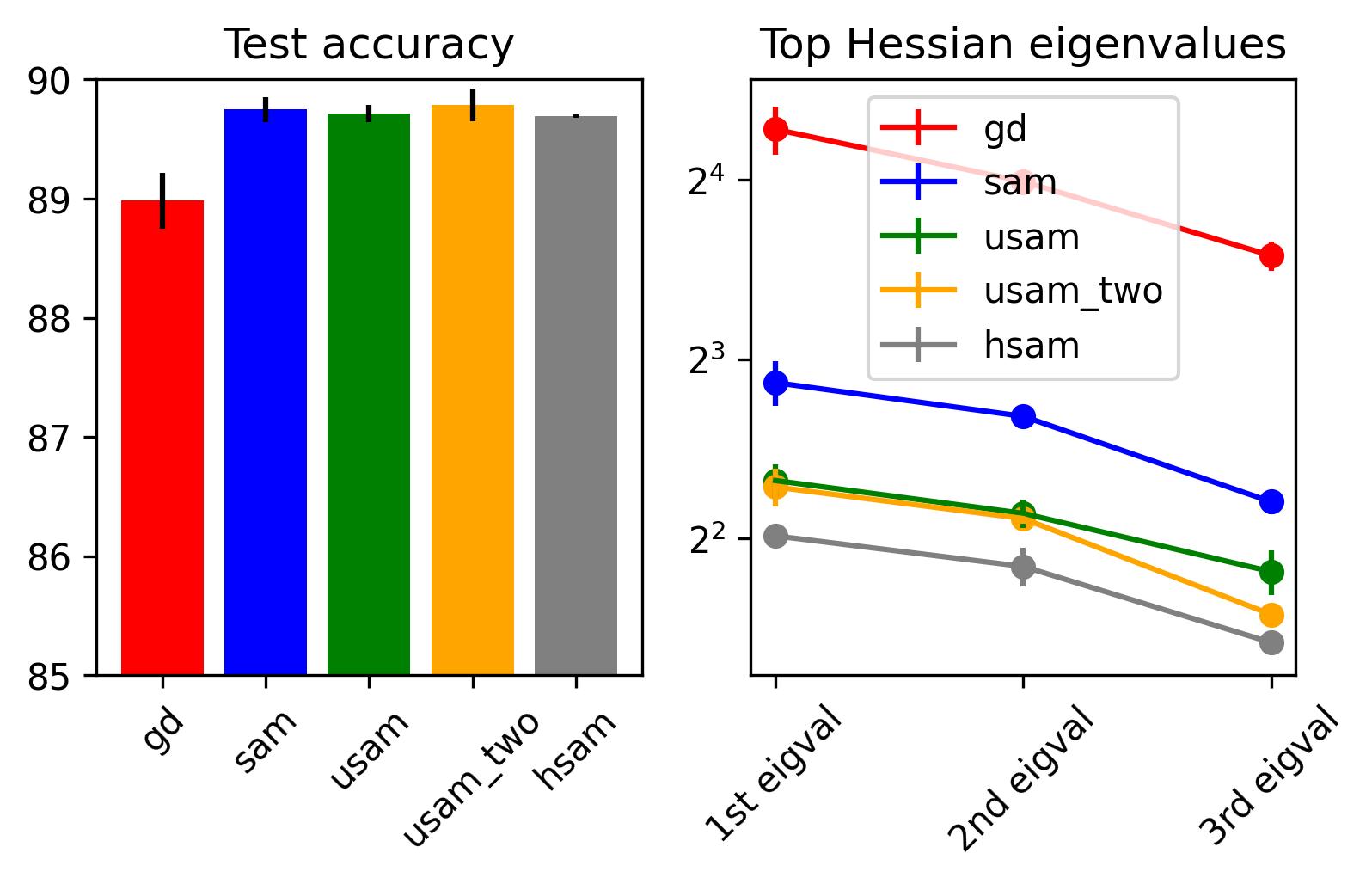}
    \caption{(\textbf{Experiment 2}) - Same as \Cref{fig:mnist-sgd} with the MNIST-FASHION dataset.} 
    \label{fig:mnist-fashion-sgd}
\end{figure}

\section{Proof of \Cref{theorem:learning-rate-gradient-stability}}
\label{sec:semi-algebraic}
We first provide preliminary lemmas and then proceed to the proof of \Cref{lemma:gradient-update-good} and \Cref{theorem:learning-rate-gradient-stability}.
\begin{lemma}
    \label{lemma:definability-preserves-zero-sets}
    Let $F: \RR^m \to \RR^m$ be a \sa function and $S \subset \RR^m$ be a $C^1$ embedded submanifold of dimension at most $m-1$. Then, $F(S)$ is contained in a countable union of $C^1$ embedded submanifolds of dimension at most $m-1$.
\end{lemma}
\begin{proof}
	Invoking \cite[Lemma C.2]{o-minimal-structures} there exists $P_1, \ldots, P_N$, disjoint, \sa and open, whose union is dense in $\RR^m$, such that for each $k = 1,\ldots, N$, the restriction $F_k = F \mid_{P_k}: P_k \to \RR^m$ is $C^1$ and $J_{k}:= \jac {F_k}$ has constant rank. Set $P_0 = \cap_{k=1}^N P_k^c$, $P_0$ is \sa of dimension at most $m-1$. Partition the set $S$ into:
    \begin{equation*}
        S_k = P_k \cap S, \qquad \forall k = 0, 1, \ldots, N.
    \end{equation*}
	We have that $\cup_{k=0}^N P_k = \RR^m$ so that $F(S) = \cup_{k=0}^m F_k(S_k)$. Consider three cases:
    \begin{enumerate}[leftmargin=*]
        \item $k = 0$: $F(S_0) \subset F(P_0)$ which is of dimension at most $m-1$ by \cite[4.7]{o-minimal-structures} so that $F(P_0)$ is contained in a finite union of $C^1$ embedded submanifolds of dimension at most $m-1$ \cite[4.8]{o-minimal-structures}.
        \item $k > 0$ and $\rank(J_k) < m$: Then $S_k \subseteq P_k$ is a subset of the critical points of the mapping $F$. Since $F_k: P_k \to \RR^m$ is $C^1$, we can apply the Sard theorem to conclude that $\mu(F(P_k)) = 0$ ($\mu(\cdot)$ is the Lebesgue measure). Since $F(P_k)$ is also \sa (because $F$ and $P_k$ are \sa), its zero Lebesgue measure implies that $F(S_k) \subseteq F(P_k)$ is contained in a finite union of $C^1$ embedded submanifolds. 
        \item If $k > 0$ and $\rank(J_k) = m$: then $F_k$ is a local diffeomorphism. It implies that for any $x \in S$, there exists an open neighborhood $V_x \subseteq P_k$ such that $F(S \cap V_x)$ is an embedded submanifold of dimension at most $m-1$. By taking a countable open covering $\{V_i, i \in \NN\}$ of $S_k$, we have:
        \begin{equation*}
            F_k(S_k) \subseteq \bigcup_{i \in \NN} F(S \cap V_i),
        \end{equation*}
    \end{enumerate}
    and hence, the result.
\end{proof}

\begin{lemma}
    \label{lemma:definable-selection}
	Consider a \sa set $S \subset \RR^n \times \RR^m$. If for all $x \in \RR^n$, the fiber $S_x = S \cap \{x\} \times \RR^m$ only contains isolated points. Then there exists an integer $N$ and $N$ \sa functions $F_1, \ldots, F_N \colon \proj_{\RR^n} S \to \RR^m$, $k = 1, \ldots, N$ such that:
    \begin{equation}
        \label{eq:partition-definable-set-to-functions}
        S = \bigcup_{k = 1}^N \graph F_k. 
    \end{equation}
\end{lemma}
\begin{proof}
	By \cite[Properties 4.4]{o-minimal-structures}, the number of connected components of $S_x$ is uniformly bounded by a number $N \in \NN$. Moreover, the connected components are singletons because they only consist of isolated points. Therefore, there exists a positive integer $N$ such that $|S_x| < N$, for all $x \in \RR^n$.

	We construct the \sa functions $F_1, \ldots, F_N$ recursively as follows. Set $S_1 = S$ by \cite[Property 4.5]{o-minimal-structures}, there is a \sa function $F_1 \colon \proj_{\RR^n} S_1 \to \RR^m$ such that $\graph F_1 \subset S_1 = S$. We set $\cD = \proj_{\RR^n} S$, the domain of $F_1$. Recursively, we set for $k \geq 2$, $S_k = S_{k-1} \setminus \graph F_{k-1}$ and define similarly $F_k \colon \proj_{\RR^n} S_k \to \RR^m$ such that $\graph F_k \subset S_k \subset S$. At each iterations the cardinality of the fibers of $S_k$ is reduced by at least $1$ compared to those of $S_{k-1}$. After $N$ iterations, we have $S_N \setminus \graph F_N = \emptyset$ and $S = \bigcup_{k = 1}^N \graph F_k$. Each function can be extended to the whole set $\cD$ by choosing the value $F_k(x) = F_1(x)$ outside of the domain of definition of $F_k$. This preserves \sa sets as well as the equality; this concludes the proof.
\end{proof}

\begin{proof}[Proof of \Cref{lemma:gradient-update-good}]
    It is sufficient to prove the result for $D = I$. Indeed, we can rewrite $G_\alpha = D (x - \alpha D^{-1}g(x))$. If we find $\Lambda$ satisfy \Cref{lemma:gradient-update-good} for $\bar{G}_\alpha := x - \alpha D^{-1}g(x)$, then the same $\Lambda$ also works for $G_\alpha$ since $G_\alpha$ is a composition of a global diffeomorphism $x \mapsto D x$ and $\bar{G}_\alpha$. Therefore, in the following, we can assume that $D = I$, i.e. $G_\alpha(x) = x - \alpha g(x)$.

	Consider $\lambda^\cR(x): \RR^m \to \RR^m: (\lambda^\cR_i)_{i = 1}^m$, the real parts of eigenvalues of the Jacobian matrix $\jac g(x)$, counted with their multiplicity. Since $\lambda^\cR_i, i = 1, \ldots, m$ are \sa, we choose a \sa, open, and dense subset $I \subset \RR^m$ so that $\lambda^\cR_i, i = 1, \ldots, M$ are all differentiable on $I$.

	We define $\Lambda := \{\alpha > 0 \mid \alpha^{-1} \notin \cup_{i = 1}^m \lambda^\cR_i(\crit \lambda^\cR_i)\}$. We prove that $\Lambda$ satisfies the conditions of \Cref{lemma:gradient-update-good}. Due to the \sa Sard's theorem \cite{kurdyka2000semialgebraic}, the set of critical values of $\lambda^\cR_{i}$ is of zero Lebesgue measure. Moreover, it is also \sa. Hence, $\cup_{i = 1}^m \lambda^\cR_i(\crit \lambda^\cR_i)$ is finite. Thus, the complement $\RR_{>0} \setminus \Lambda$ is finite. Fix $\alpha \not \in \Lambda$, we are going to verify the required condition.
	
	Set $K_\alpha = \{x \in \RR^m \mid \det(I - \alpha \jac g(x)) = 0\}$, which is \sa. We are going to show that $\dim(K_\alpha) < m$.
	Partition $K_\alpha$ into two sets:
    \begin{equation*}
        K_1 = K_\alpha \cap I \qquad \text{ and } \qquad K_2 = K_\alpha \cap I^c,
    \end{equation*}
	where $I$ is a \sa dense open set in which all functions $\lambda_i^{\cR}$, $i=1,\ldots,m$ are differentiable. Since $I$ is open, dense and \sa, $\dim(I^c) < m$ and thus, $\dim(K_2) < m$. To prove that $\dim(K_1) < m$, we notice that:
    \begin{equation*}
        K_\alpha \subseteq \cup_{i = 1}^m K_{\alpha, i} \quad \text{ where } \quad K_{\alpha, i} := \{x \in I \mid \alpha\lambda_i^{\cR}(x) = 1\}.
    \end{equation*}
    The sets $K_{\alpha,i}, i = 1, \ldots, m$ are \sa themselves. Their dimension has to be strictly smaller than $m$. Indeed, by contradiction, if there exists $i$ such that $K_{\alpha,i}$ has dimension $m$, it must contain an open set $O$. For all $x \in O$, $\lambda_i^{\cR}(x) = 1/\alpha$ is a constant. Therefore, $1 / \alpha$ is a critical value of $\lambda_i^{\cR}$, which contradicts our choice of $\alpha$. Thus, $K_{\alpha,i}$ and $K_\alpha$ have dimension strictly smaller $m$.
    
	Given an embedded $C^1$ submanifold $S$ of dimension $p<m$, we partition its pre-image by $G_\alpha$, $\alpha \in \Lambda$ as:
    \begin{equation*}
        S_1 = G_\alpha^{-1}(S) \cap K_\alpha \qquad \text{ and } \qquad S_2 = G_\alpha^{-1}(S) \cap K_\alpha^c.
    \end{equation*}
    Since $K_\alpha$ is \sa of dimension at most $m-1$, it is contained in a finite union of submanifolds of dimension at most $m-1$. It is sufficient to show that the same property holds for $S_2$ as well. Consider the following set:
    \begin{equation*}
        T = \{(y,x) \mid y = x - \alpha g(x) \text{ and } x \in K_\alpha^c\} \subseteq \RR^{2m}.
    \end{equation*}
		$S_2$ is the projection of $T \cap S \times \RR^m$ onto the last $m$ coordinates. The set $T$ is \sa (although $S_2$ is generally not since we do not assume that $S$ is \sa). In particular, for any $(y,x) \in T$, the following equation is satisfied $f(x, y) = y  - x + \alpha g(x) = 0$. In addition, $x \in K_\alpha^c$, so that $\frac{\partial}{\partial x} f(x,y) = - I + \alpha \jac g (x)$ which is invertible since $x \in K_\alpha^c$. By the implicit function theorem, we can conclude that the fiber $T_y := \{x \mid (y,x) \in T\}$ contains isolated points. And since $T_y$ is \sa, $T_y$ is finite. We are, thus, in the position to apply \Cref{lemma:definable-selection}: There exists $N$ \sa $F_1, \ldots, F_N, N > 0$ such that:
    \begin{equation*}
        T \subseteq \bigcup_{k = 1}^N \graph F_k. 
    \end{equation*}
    In particular, this relation implies:
    \begin{equation*}
        S_2 \subseteq \bigcup_{k = 1}^N F_k(S).
    \end{equation*}
    The result follows by \Cref{lemma:definability-preserves-zero-sets}.
\end{proof}

We conclude this section with the proof of the main theorem, following standard arguments.
\begin{proof}[Proof of \Cref{theorem:learning-rate-gradient-stability}]
    Take $\Lambda$ as in \Cref{lemma:gradient-update-good}, whose complement is finite. We prove that this $\Lambda$ is our desired set described as in \Cref{theorem:learning-rate-gradient-stability}. 

	Fix $\alpha \in \Lambda$ and set $ \cC_\alpha = \{x \in \RR^n \mid G_\alpha(x) = x,\, \rho(\jac G_\alpha(x)) > 1\}$.
    For each $x \in \cC_\alpha$, let $B_x$ be the open balls corresponding to $x$ given by \Cref{theorem:refined-center-stable-manifold}. In particular, we have:
    \begin{equation*}
        \cC_\alpha \subseteq \bigcup_{x \in \cC_\alpha} B_x. 
    \end{equation*}
	By Lindelof’s lemma, there exists a sequence $(z_\ell)_{\ell \in \NN}$ in $\cC_\alpha$, such that $\cC_\alpha \subseteq \cup_{\ell \in \NN} B_{z_\ell}$.
    
	Assume that the update \Cref{eq:iterative-methods} initialized at a point $x_0$ and converges to a point $x \in \cC_\alpha$. Thus, there exists natural numbers $\ell, k_0 \in \NN$ such that for $k \geq k_0$, $G_\alpha^k(x_0) \in B_{z_\ell}$. In particular, in the light of \Cref{theorem:refined-center-stable-manifold}, $G_\alpha^{k_0}(x_0) \in W^{\text{cs}}_{z_\ell}$, the local stable center manifold of $z_\ell$, given by \Cref{theorem:refined-center-stable-manifold}. Since $x \in \cC_\alpha$ was arbitrary, this shows that the set $W_\alpha$ in the statement of the theorem has the following properties
    \begin{equation*}
        W_\alpha \subseteq \bigcup_{\ell \in \NN} \bigcup_{k \in \NN} G_\alpha^{-k}(W^{\text{cs}}_{z_\ell}).
    \end{equation*}
	For each $\ell \in \NN$, $W^{\text{cs}}_{z_\ell}$ is a $C^1$ embedded submanifold. Using \Cref{lemma:definability-preserves-zero-sets}, we see that $W_\alpha$ is contained in a countable union of $C^1$ embedded submanifolds as announced.
\end{proof}

\section{Proof of \Cref{theorem:refined-center-stable-manifold}}
\label{sec:center-stable-manifold}
The main technical tool to obtain this result is a stability theorem related to pseudo hyperbolicity \cite{hirsch1977invariant}, which we report for Euclidean spaces. We start with the definition of pseudo-hyperbolicity.

\begin{definition}[$\rho$-pseudo hyperbolicity]
    A linear map $T: \RR^n \to \RR^n$ is $\rho$-pseudo hyperbolic if all eigenvalues of $T$ have absolute values different from $\rho > 0$.
    Suppose $T$ is $\rho$-pseudo hyperbolic. In that case, we define $\RR^n = E_{sc} \oplus E_u$ the canonical splitting of $T$ where $E_{sc}$ (resp. $E_u$) is the linear subspace induced by the eigenvectors corresponding to the eigenvalues with absolute values smaller (resp. bigger) than $\rho$.
\end{definition}

\begin{theorem}[Stable manifold for pseudo hyperbolic maps {\cite[Theorem 5.1]{hirsch1977invariant}}]
    \label{lemma:Lipschitz-pertubation-manifold}
    Consider $T: \RR^n \to \RR^n$ a $\rho$-pseudo hyperbolic linear map and its canonical splitting $\RR^n = E_{sc} \oplus E_u$. If $F: \RR^n \to \RR^n$ is a $C^1$ map, $F(0) = 0$ and the function $F - T$ has a sufficiently small Lipschitz constant $\epsilon$, then the set:
    \begin{equation*}
        W = \bigcap_{k \geq 0} F^{-k} S, \qquad S = \{(x,y) \in E_{u} \times E_{sc}: \|y\| \geq \|x\|\}
    \end{equation*}
    is the graph of a $C_1$ function $g: E_{sc} \to E_u $. It is characterized by: $z \in W$ if and only if: 
    $$\lim_{k \to \infty} \|F^k(z)\| / \rho^k = 0.$$
\end{theorem}

Given \Cref{lemma:Lipschitz-pertubation-manifold}, one can adapt localization arguments, such as \cite[Chapter 5, Theorem III.7]{shub1987global}, to prove \Cref{theorem:refined-center-stable-manifold}.
\begin{proof}[Proof of \Cref{theorem:refined-center-stable-manifold}]
    We may assume that $p = 0$ by studying $x \mapsto F(x + p) - p$ instead of $F$. Consider a $C^\infty$ function $\varphi(x): \RR^n \to \RR$ satisfying:
    \begin{equation*}
        \begin{cases}
            \varphi(x) = 1 & \text{if } \|x\| \leq 1\\
            \varphi(x) = 0 & \text{if } \|x\| \geq 2\\
        \end{cases}.
    \end{equation*}
    Such a function exists and it is known as a bump function. In particular, if one defines: $\varphi_s(\cdot) = \varphi(\cdot / s), s > 0$, then the function $\varphi_s$ is smooth and satisfies:
    \begin{equation*}
        \begin{cases}
            \varphi_s(x) = 1 & \text{if } \|x\| \leq s\\
            \varphi_s(x) = 0 & \text{if } \|x\| \geq 2s\\
        \end{cases}.
    \end{equation*}

	Let $T$ denote the linear mapping $x \mapsto \jac F(0)x$,
	define $h = F - T$,  $h_s  = \varphi_s \times h$ and $F_s = T+h_s$. We have for any $s>0$, $h_s = h$ and $F_s = F$ in the ball of radius $s$. One may choose $s>0$ such that the Lipschitz constant of $h_s = F_s - T$ is arbitrarily small. Indeed, we have for all $x$:
    \begin{equation*}
		\jac h_s(x) = h(x) \nabla \varphi_s(x)^T  + \varphi_s(x) \jac h(x) = \frac{h(x)}{s}\nabla \varphi\left(\frac{x}{s}\right)^T  + \varphi_s(x) \jac h(x).
    \end{equation*}
	We remark that for any $s>0$, $\jac h_s(x) = 0$ for any $x$ such that $\|x\|_2 \geq 2s$, we will obtain a bound on a ball of radius $2s$ for $s$ small. Since $h$ is $C^1$ and $\jac h(0) = 0$, we have:
    \begin{align*}
		\lim_{s \to 0} \sup_{\|x\| \leq 2s} \frac{\|h(x)\|}{s} &= 0,\\
		\lim_{s \to 0} \sup_{\|x\| \leq 2s} \|\jac h(x)\| &= 0.
    \end{align*}
    Since $\varphi$ is smooth and constant outside a ball, both $\varphi_s$ and $\nabla \varphi$ are globally bounded. Therefore, we conclude that: for any $\epsilon > 0$, there exists $s > 0$ such that:
    \begin{equation*}
        \|\jac h_s(x)\| \leq \epsilon, \forall x \in \RR^n,
    \end{equation*}
	which implies that $F_s - T$ is $\epsilon$ Lipschitz. We fix {$1 < \rho < \min\{|\lambda| \mid \lambda \text{ is an eigenvalue of } \jac F(0), |\lambda| > 1\}$} 
    and apply \Cref{lemma:Lipschitz-pertubation-manifold} to conclude that there exists a $C^1$ function $g:  E_{sc} \to E_u$ such that: $x \in \graph g$ if and only if:
    \begin{equation}
        \label{eq:invariant-characterization}
        \lim_{k \to \infty} \frac{\|F_s^k(x)\|}{\rho^k} = 0,
    \end{equation}
	We define $B := B(0, s)$ and $W^{\text{sc}}_{\text{loc}} = \graph g \cap B$, which satisfies the requirements of \Cref{theorem:refined-center-stable-manifold}. First, $\graph g$ is $F_s$ invariant from the characterization \Cref{eq:invariant-characterization} and $F_s \mid_{B} = F \mid_B$ hence 
    $$F(\graph g \cap B) \cap B = F_s(\graph g \cap B) \cap B \subset \graph g \cap B,$$ 
    which proves the first property. Second, if $F^k(x) \in 
    B$ for all $k \in \NN$, 
	then
	\begin{equation*}
		{\lim }_{k\to \infty}\frac{\|F_s^{k}(x)\|}{\rho^{k}} = {\lim }_{k\to \infty}\frac{\|F^{k}(x)\|}{\rho^{k}} = 0,
    \end{equation*}
    since $\rho > 1$ and $F^k(x)$ is bounded, that is $x \in \graph g \cap B$. The proof is concluded. 
\end{proof}

\section{Details on eigenvalue computation}
\label{sec:algorithms}
\begin{proof}[Proof of \Cref{cor:application-gradient-descent}]
    We apply \Cref{theorem:learning-rate-gradient-stability} with $D = I$ and $g = \nabla f$. Let $\lambda_1, \ldots, \lambda_m$ be the eigenvalues of $\nabla^2 f(\bar{x})$, then the eigenvalues of $D - \alpha \jac g(\bar{x})$ is given by:
    \begin{equation*}
        \{1 - \alpha \lambda_i \mid i = 1, \ldots, m\}
    \end{equation*}
    Therefore, $\rho(D - \alpha \jac g(\bar{x})) \leq 1$ is equivalent to:
    \begin{equation*}
        -1 \leq 1 - \alpha \lambda_i \leq 1, \forall i \qquad \Leftrightarrow \qquad 0 \leq \lambda_i \leq \frac{2}{\alpha}, \forall i. 
    \end{equation*}
\end{proof}
\begin{proof}[Proof of \Cref{cor:application-momentum-gradient}]
    The proof consists of calculating the eigenvalues of the following matrix:
    \begin{equation*}
        A := \begin{pmatrix}
            (1 + \beta)I - \alpha \nabla^2 f(x) & -\beta I \\
            I & {0}
        \end{pmatrix}
    \end{equation*}
    We reproduce the arguments in \cite[Chapter $3.2$]{polyak1987introduction}.
    
    It can be shown that for any eigenvalue $\lambda$ of $\nabla^2 f(x)$, there is a corresponding pair of eigenvalues of $A$, which are the zero of the following quadratic equation :
    \begin{equation*}
        \nu^2 - \nu (1 + \beta - \alpha\lambda) + \beta = 0.
    \end{equation*}
    To make sure that all eigenvalues of $A$ have norm at most one, it is necessary and sufficient that:
    \begin{equation*}
        0 \leq \alpha \lambda \leq 2(1 + \beta) 
    \end{equation*}
    for every eigenvalue $\lambda$ of $\nabla^2 f(x)$, which proves the result.
\end{proof}
\begin{proof}[Proof of \Cref{cor:application-nesterov-gradient}]
    The proof of \Cref{cor:application-nesterov-gradient} is similar to the proof of \Cref{cor:application-momentum-gradient}. However, since it seems to us that the calculation of eigenvalues for the Jacobian matrix of the iteration update \Cref{eq:nesterov-accelerated-gradient} is less known, we provide a detailed derivation.

    Note that $(\bar{x},\bar{y})$ is a fixed point of \Cref{eq:nesterov-accelerated-gradient} if and only if $\bar{x} = \bar{y}$ and $\nabla f(\bar{x}) = 0$. We derive the Jacobian matrix of \Cref{eq:nesterov-accelerated-gradient} as:
    \begin{equation*}
        \begin{pmatrix}
            (1 + \beta)(I - \alpha \nabla^2 f(\bar{x})) & -\beta (I - \alpha \nabla^2 f(\bar{x}))\\
            I & 0
        \end{pmatrix},
    \end{equation*}
    hence, given an eigenvalue of $\lambda$ of $\nabla^2 f(\bar{x})$, the previous Jacobian matrix possesses two corresponding eigenvalues of the following $2 \times 2$ matrix:
    \begin{equation*}
        \begin{pmatrix}
            (1 + \beta)(1 - \alpha \lambda) & -\beta (1 - \alpha\lambda) \\
            1 & 0 
        \end{pmatrix}.
    \end{equation*}
    These eigenvalues are given by the roots of the following quadratic equations:
    \begin{equation*}
        \nu^2 - \nu \underbrace{(1 + \beta)(1 - \alpha\lambda)}_{:=b} + \underbrace{\beta (1 - \alpha\lambda)}_{:=c} = 0.
    \end{equation*}
    Consider two possibilities concerning $\Delta:   = b^2 - 4c = (1 + \beta)^2(1 - \alpha\lambda)^2 - 4\beta(1-\alpha\lambda)$.
    \begin{enumerate}[leftmargin=*]
        \item $\Delta < 0$: this condition implies that $\alpha \lambda \in [(\beta - 1)^2/(\beta + 1)^2, 1]$. When $\Delta < 0$, the quadratic equation admits two complex roots whose magnitude is given by:
        \begin{equation*}
            \frac{1}{4}(b^2 + 4c - b^2) = c.
        \end{equation*}
        Hence, the condition of \Cref{theorem:learning-rate-gradient-stability} reads as: 
        \begin{equation*}
            \beta (1 - \alpha \lambda) \leq 1 \implies \alpha\lambda \geq 1 - \frac{1}{\beta}
        \end{equation*}
        Note that when $0 < \beta < 1$, $1 - \frac{1}{\beta} < 0$. Therefore, the previous equation is satisfied automatically because $\alpha \lambda \geq (\beta - 1)^2/(\beta + 1)^2 \geq 0$.
        \item When $\Delta \geq 0$, the quadratic equation has two real roots, given by:
        \begin{equation*}
            \frac{b - \sqrt{\Delta}}{2} \qquad \qquad \text{and} \qquad \qquad \frac{b + \sqrt{\Delta}}{2}. 
        \end{equation*}
        Thus, the condition of \Cref{theorem:learning-rate-gradient-stability} reads as:
        \begin{equation*}
            \begin{aligned}
                \frac{b - \sqrt{\Delta}}{2} & \geq -1,\\
                \frac{b + \sqrt{\Delta}}{2} & \leq 1.\\
            \end{aligned}
        \end{equation*}
        Solving both inequalities:
        \begin{equation*}
            \begin{aligned}
                \frac{b - \sqrt{\Delta}}{2} \geq -1 &\implies b + 2 \geq \sqrt{\Delta}\\
                &\implies b^2 + 4b + 4 \geq \Delta = b^2 - 4c \\
                &\implies (1 + \beta)(1 -\alpha \lambda) + 1 \geq -\beta (1 -\alpha\lambda)\\
                &\implies (1 + 2\beta)(1 -\alpha\lambda) + 1 \geq 0\\
                &\implies \frac{2 + 2\beta}{1 + 2\beta} \geq \alpha\lambda.
            \end{aligned}
        \end{equation*}

        \begin{equation*}
            \begin{aligned}
                \frac{b + \sqrt{\Delta}}{2} \leq 1 &\implies 2 - b \geq \sqrt{\Delta}\\
                &\implies b^2 - 4b + 4 \geq \Delta = b^2 - 4c \\
                &\implies (1 + \beta)(1 -\alpha \lambda) - 1 \leq \beta (1 -\alpha\lambda)\\
                &\implies (1 -\alpha\lambda) - 1 \leq 0\\
                &\implies \alpha\lambda \geq 0.
            \end{aligned}
        \end{equation*}
    \end{enumerate}
    By combining all these observations, we conclude that $0 \leq \alpha \lambda \leq \frac{2 + 2\beta}{1 + 2\beta}$, which yields the result.
\end{proof}

\begin{proof}[Proof of \Cref{cor:application-usam-momentum}]
    We only consider the case $\beta = 0$ (cf. \Cref{eq:unnormalized-sam}). The computation similarly extends to general $\beta$ as in \Cref{cor:application-momentum-gradient}.
    
    We compute the Jacobian matrix the update rule \Cref{eq:unnormalized-sam}:
    \begin{equation*}
        I - \alpha\nabla^2 f(\bar{x} + \rho \underbrace{\nabla f(\bar{x})}_{= 0})(I + \rho \nabla^2 f(\bar{x})) = I - \alpha\nabla^2 f(\bar{x})(I + \rho \nabla^2 f(\bar{x})) 
    \end{equation*}
    and their eigenvalues, given by:
    \begin{equation*}
        1 - \alpha \lambda(1 + \rho \lambda),
    \end{equation*}
    where $\lambda$ is an eigenvalue of $\nabla^2 f(\bar{x})$.
    Applying \Cref{theorem:learning-rate-gradient-stability} simply yields the result. 
\end{proof}

\begin{proof}[Proof of \Cref{cor:application-usam-momentum-two}]
    Similar to the proof of \Cref{cor:application-usam-momentum}, we only consider the case $\beta = 0$ (cf. \Cref{eq:unnormalized-sam-sam}) and the computation for general $\beta$ (cf. \Cref{eq:momentum-usam}) can be done similarly as in \Cref{cor:application-momentum-gradient}. Consider $\bar{x}$ a critical point of $f$. Computing the Jacobian matrix of the iteration update of \Cref{eq:unnormalized-sam-sam} yields:
    \begin{equation*}
        H := I - \alpha \nabla^2 f(\bar{x}) (1 + \rho \nabla^2 f(\bar{x}))^2.
    \end{equation*}
    Note that for an eigenvalue of $\lambda$ of $\nabla^2 f(\bar{x})$, $H$ has an eigenvalue equal to $1 - \alpha \lambda(1 + \rho\lambda)^2$. Applying \Cref{cor:application-usam-momentum} yields the result immediately.
\end{proof}

\begin{proof}[Proof of \Cref{cor:application-hsam}]
    Similar to \Cref{cor:application-usam-momentum-two}.
\end{proof}

\section{Appearance of new fixed points}\label{appearance}

\subsection{USAM}
For $K>0$, set
$$
f(x)=
\begin{cases}
0 & x\le 1\\
-\dfrac{K}{3}\,(x-1)^3 & x>1,
\end{cases}
$$
so that \(\crit f = (-\infty,1]\) where all points are local minimizers except $x=1$.  One has $f'(x)=0$ for $x\le 1$, while $f'(x)=-K(x-1)^2$ for $x>1$.
To model USAM, set $T(x):=x+\rho f'(x)$, so that 
\[
T(x)=
\begin{cases}
x & x\le 1,\\[2pt]
x-\rho K (x-1)^2 & x>1.
\end{cases}
\]
Observe that $x\in T^{-1}(\crit f)\:\mbox{ iff }\:T(x)\le 1$. 

For $x>1$, write \(s:=x-1>0\); the condition becomes
\[
1+s-\rho K s^2 \le 1
\;\;\Longleftrightarrow\;\;
s(1-\rho K s)\le 0
\;\;\Longleftrightarrow\;\;
s \ge \frac{1}{\rho K}.
\]
Hence
$\displaystyle
T^{-1}(\crit f)\cap(1,+\infty)=\big[\,1+\tfrac{1}{\rho K},\,+\infty\big).$ This shows that USAM creates an infinite length interval of fixed points.

\subsection{Two steps USAM}
\label{sec:spuriousExampleTwoStepUSAM}
{Set $\rho = 1$ and $\beta = 0$ and consider the univariate function $f \colon x \mapsto -\frac{x^2}{2} - \frac{x^3}{3}$. $0$ is a strit local maximum with $f''(0) = -1 = -1/\rho$. Writing $T(x) = x + f'(x) = -x^2$, the two step USAM algorithm has an explicit update as follows
\begin{align*}
    x_{k+1} = x_k - \alpha(x_k^4 - x_k^8).
\end{align*}
For small enough positive initialization, these iterations monotically converge to $0$.
}

\section{Experiments on deterministic algorithms}
\label{appendix:extra-experiments}
We present an additional experiment in this section. We perform one experiment from \cite[Appendix J.4]{cohen2021gradient}, which trains a subset of $5000$ images of the dataset CIFAR10 with a ReLU MLP (we used the same architecture as in \cite[Appendix I.1]{cohen2021gradient}). This training is conducted with four \emph{deterministic} algorithms: Gradient Descent, USAM, USAM2, and Hessian USAM, unlike the experiments in the main text, where all algorithms are \emph{stochastic}. All the algorithms are run with a learning rate equal to $0.01$. USAM is tune with three $\rho \in \{0.05, 0.1, 0.2\}$ while USAM2 and Hessian SAM are tuned with $\rho \in \{0.001, 0.002, 0.005\}$. For each algorithm, the results with $\rho$ giving the best validation accuracy will be chosen and reported. For each algorithm and its corresponding $\rho$, we run $10^4$ (full-batch) iterations and repeat the process three times to record the average and variance of the top eigenvalues of the Hessian at the end of training. The results are presented in \Cref{fig:deterministic-algo}.

\begin{figure}[h!]
    \centering
    \includegraphics[width=0.8\linewidth]{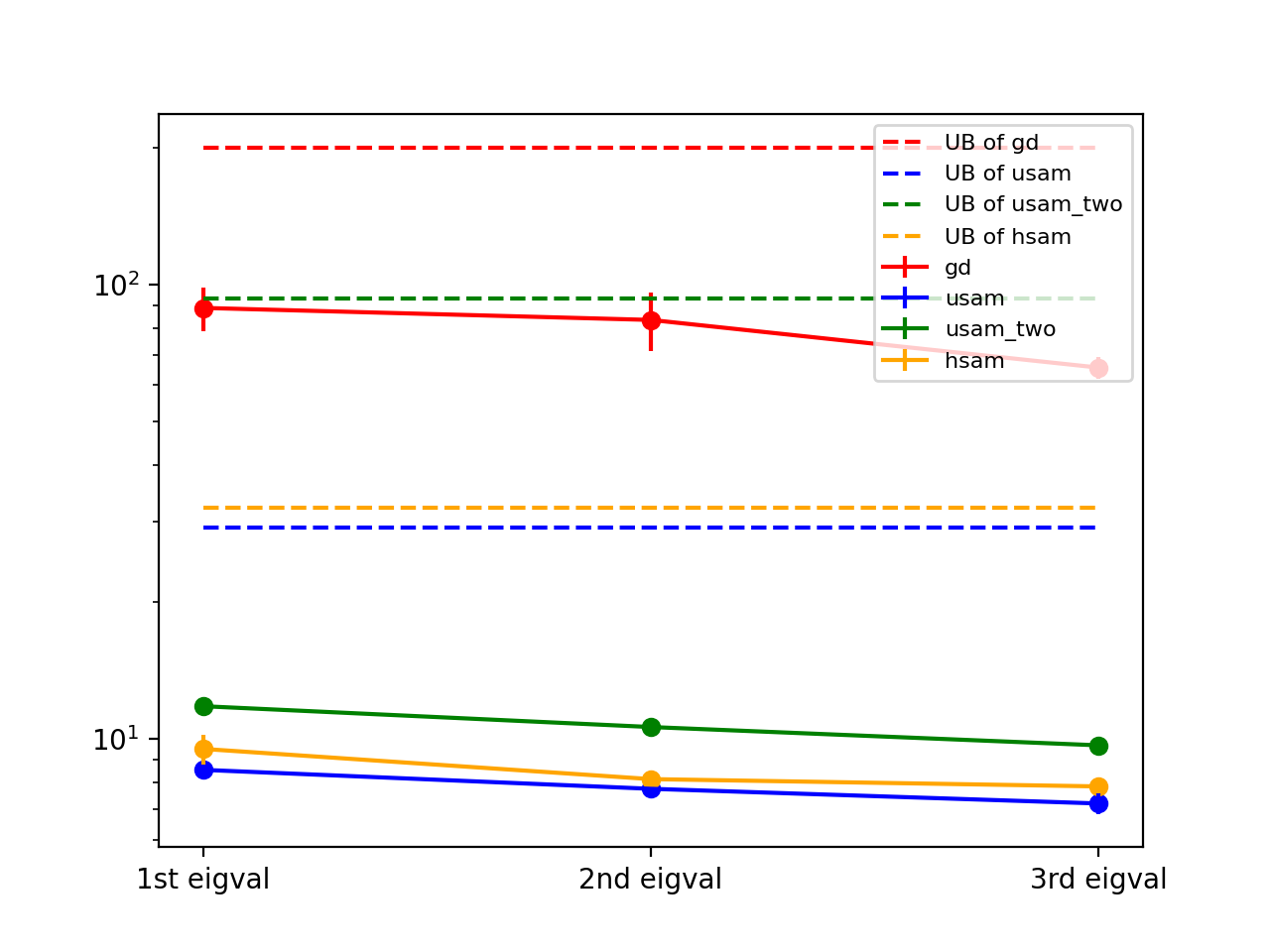}
    \caption{Theoretical upperbounds (UB) given by \Cref{cor:application-gradient-descent,cor:application-usam-momentum,cor:application-usam-momentum-two,cor:application-hsam} and the actual top three eigenvalues of Gradient descent, USAM, USAM2, and Hessian SAM at the end of the training.}
    \label{fig:deterministic-algo}
\end{figure}
We remark that the upperbounds derived in the main text are validated by our experiments. The actual top eigenvalues of the Hessian of the last iterates of four experimented algorithms sit well below their thresholds. No saturation of eigenvalues, which is similar to what is reported in \cite[Appendix A]{cohen2021gradient}: with the cross-entropy loss function, it is harder to reproduce the EoS phenomenon. Finally, we observe that while the magnitudes of the eigenvalues of the four algorithms do not exactly equal to their upperbounds, both share the same order of magnitude in this experiment. This suggests that the previously derived upper bounds can serve as a good indicator for selecting algorithms that find the flattest local minima in practical neural network training.  
\end{document}